\documentclass[11pt]{article}

\usepackage{dsfont}

\usepackage{environ}

\usepackage{amsthm,amsfonts,amsmath,amssymb,epsfig,color,float,graphicx,verbatim, enumitem}
\usepackage{multirow}
\usepackage{algorithm}
\usepackage{algorithmicx}
\usepackage{algpseudocode}
\usepackage{caption}
\usepackage{subcaption}
\usepackage{threeparttable}

\usepackage{hyperref}
\hypersetup{
  colorlinks = true,
  urlcolor = {cpurple},
  citecolor = {cpurple}
}
\definecolor{cpurple}{rgb}{0.6,0,0.6}

\usepackage{enumitem}

\newtheorem{theorem}{Theorem}[]
\newtheorem{question}{Question}[]

\newtheorem{lemma}{Lemma}[]
\newtheorem{informal theorem}[theorem]{Theorem (informal statement)}

\newtheorem{proposition}[theorem]{Proposition}

\newtheorem{claim}[theorem]{Claim}
\newtheorem{fact}[theorem]{Fact}

\newtheorem{remark}[theorem]{Remark}

\newtheorem{definition}{Definition}
\newtheorem{assumption}{Assumption}[]

\newcommand{\lp}{\left}
\newcommand{\rp}{\right}
\newcommand\norm[1]{\left\| #1 \right\|}
\newcommand\snorm[2]{\left\| #2 \right\|_{#1}}
\newcommand\abs[1]{\left| #1 \right|}
\renewcommand\vec[1]{\boldsymbol{#1}}
\newcommand\matr[1]{\boldsymbol{#1}}

\DeclareMathOperator*{\E}{\mathbf{E}}

\newcommand{\proj}{\mathrm{proj}}

\newcommand{\normal}{\mathcal{N}}

\newcommand{\mcal}{\mathcal}

\DeclareMathOperator*{\argmin}{argmin}

\newcommand{\bx}{\mathbf{x}}

\newcommand{\bw}{\mathbf{w}}

\newcommand{\R}{\mathbb{R}}

\newcommand{\eps}{\epsilon}
\newcommand{\dtv}{d_{\mathrm TV}}

\newcommand{\poly}{\mathrm{poly}}

\newcommand{\D}{\mathcal{D}}

\newcommand{\Ind}{\mathds{1}}
\newcommand{\1}{\Ind}

\newcommand{\wt}{\widetilde}

\newcommand{\XX}{\vec{x}\vec{x}^T}
\newcommand{\SIGMA}{\matr{\Sigma_*}}

\newcommand{\GEN}{\mathcal{G}}
\newcommand{\DIS}{\mathcal{D}}

\newcommand{\LG}{\mathcal{L}_{\mathcal{G}}}
\newcommand{\LD}{\mathcal{L}_{\mathcal{D}}}
\newcommand{\VV}{\mathcal{V}}

\newcommand{\invfdist}{ \snorm{F}{ \matr{I} - \matr{W}^T \SIGMA^{-1} \matr{W} } }

\newcommand{\wsqr}{ \lp(  \matr{W}\matr{W}^T  \rp) }
\newcommand \invtrans[1]{ ( #1 ^{-1})^T }
\newcommand\flatten[2]{ \begin{bmatrix}\lp(#1\rp)^{\flat}\\#2\\\end{bmatrix}  }

\newcommand{\ReLU}{\text{ReLU}}
\newcommand{\Var}{\mathrm{Var}}

\title{Convergence and Sample Complexity of SGD in GANs}
\author{
Vasilis Kontonis\\
University of Wisconsin-Madison\\
{\tt kontonis@wisc.edu }\\
\and
Sihan Liu\\
University of Wisconsin-Madison\\
{\tt sliu556@wisc.edu}\\
\and
Christos Tzamos\\
University of Wisconsin-Madison\\
{\tt tzamos@wisc.edu}
}
\date{}

\begin{document}
\onecolumn
\maketitle
\begin{abstract}

We provide theoretical convergence guarantees on training Generative Adversarial Networks (GANs) via SGD.  We consider learning a target distribution modeled by a 1-layer Generator network with a non-linear activation function $\phi(\cdot)$ parametrized by a $d \times d$ weight matrix $\vec W_*$, i.e., $f_*(\vec x) = \phi(\vec W_* \vec x)$.

Our main result is that by training the Generator together with a Discriminator according to the Stochastic Gradient Descent-Ascent iteration proposed by Goodfellow et al. yields a Generator distribution
that approaches the target distribution of $f_*$. Specifically, we can learn the target distribution within total-variation distance $\eps$ using $\tilde O(d^2/\eps^2)$ samples which is (near-)information theoretically optimal.

Our results apply to a broad class of non-linear activation functions $\phi$, including ReLUs and is enabled by a connection with truncated statistics and an appropriate design of the Discriminator network.
Our approach relies on a bilevel optimization framework  to show that vanilla SGDA works.

 \end{abstract}

\section{Introduction}
\subsection{Background and Motivation}
\label{sec:introduction}
Since the influential work of \cite{Goodfellow}, Generative Adversarial
Networks (GANs) have seen enormous success in diverse
applications, see, for example,
\cite{ACB17, radmc15, ArjovskyB17, JWSHZ20, ZXY18, HT2017}.
Despite their success in practice, very little is currently known about their
theoretical guarantees in terms of generalization properties and the number
of samples they require for training.  In comparison, supervised models based
on neural networks for classification, are much better understood through the
theory of VC dimension and Rademacher complexity.
One of the main reasons for the limited understanding of GANs is the fact that
their training dynamics are quite complex as they correspond to a min-max
game between two neural networks, the Generator and Discriminator.  Analyzing
such min-max games even in simple settings can be quite challenging,
\cite{Luke2017unroll, KodaliAHK17} as the
natural methods commonly used for training based on stochastic gradient
descent ascent (SGDA) fail to converge.

While there are countless versions of GANs proposed in the literature that are
often times domain specific, we focus on the original GAN formulation proposed
in \cite{Goodfellow}.
In particular, we consider the following min-max game
\begin{equation}
  \label{eq:main_min_max_game}
\min_{\mcal G} \max_{\mcal D}
\E_{\vec x \sim \mcal T} \log \mcal{D}(\vec x)
+\E_{\vec x \sim \mcal G} \log (1-\mcal{D}(\vec x))\, ,
\end{equation}
where $\mcal T$ is the true and unknown distribution, $\mcal G$ is the
Generator distribution, and $\mcal D$ is the Discriminator.  The above game
is a \emph{zero-sum} game between the Discriminator $\mcal{D}$ and the
Generator $\mcal{G}$.
Our goal is to provide  theoretical  convergence  guarantees on training such
Generative Adversarial Networks (GANs) via SGD.

\subsection{Our Contribution}

In this work, we take a learning theoretic approach to understand the convergence
and sample complexity of GANs for learning distributions corresponding to
one layer neural nets. More formally we consider the following model.

\paragraph{Model.}
We assume that the underlying target distribution
has the following form:

Samples are generated by drawing a random variable $\vec z$ from a standard $d$-dimensional Gaussian
distribution which are then transformed through a one-layer neural network.
That is,
for some unknown $d\times d$ parameter matrix $\matr W_*$ and a known function $\phi: \R^d \to \R^d$, the output is equal to $\phi(\matr W_* \vec z)$.
We denote by $p(\matr W_*, \phi)$ the distribution of the random variable
$\phi(\matr W_* \vec z)$, where $\vec z$ is drawn from a standard normal distribution.

This class of distributions corresponds
to one layer neural networks with standard separable activation functions, like ReLU and sigmoid applied on each coordinate of the output. Moreover, it also captures much more complex non-linearities as it
allows for arbitrary functions $\R^d \to \R^d$ that are given as input.
For instance, it can capture multi-layer neural networks as long as
only the parameters $\matr W_*$ of the first layer are unknown
while all others are fixed in advance.

Without any additional assumptions on the transformation function $\phi$ the problem
is information theoretically intractable. We identify a natural property
of the transformation $\phi$ that makes it possible to learn the underlying
distribution using a GAN architecture without restricting the expressiveness
of our model.

\emph{We require the transformation} $\phi : \R^d \mapsto \R^d$ \emph{to be
invertible with non-trivial probability over the samples of the true distribution $p(\matr W_*, \phi)$.}

We note that, commonly used activation functions in neural
networks are either fully invertible (e.g. sigmoid)
or partially invertible like ReLU that is invertible when the coordinates of $\vec x$ are positive.

More precisely we define the following class of transformed distributions.
\begin{definition}[Partially Invertible Network]
\label{def:invertible_network}
A pair $(\matr W_*, \phi)$ composed of a weight matrix $\matr W_*$ and an activation $\phi$ is denoted as a one layer partially invertible network if there exists some set $T\subseteq \R^d$
such that $\phi$ is invertible on $T$ and $\normal(T; \matr W_*) \geq \alpha> 0$ \footnote{We use $\normal(T; \matr W_*)$ to denote the mass of the set $T$ under the normal distribution with covariance $\matr W_* \matr W_*^T$. See Section \ref{sec:notation} for details. }
,where $\alpha = \Omega(1) $ is some absolute constant.
\end{definition}

Our main result is that Generative Adversarial Networks with Partially Invertible Generator Networks
converges to the true distribution
when trained by stochastic gradient descent ascent. In particular, simultaneous training of the Generator with an appropriately designed Discriminator succeeds in learning the target distribution in polynomially many iterations and near-optimal sample complexity.

\begin{theorem} \label{infthm:main}
Consider samples generated by a Partially Invertible Network $(\matr W_*, \phi)$ for some unknown $\matr W_*$ with bounded distance to $\matr I$:
$$
\max \lp( \|\matr W_*\matr W_*^T - \matr I\|_F, \| ( \matr W_*\matr W_*^T )^{-1} - \matr I\|_F \rp) < c.
$$
Then, Nested Stochastic Gradient Descent-Ascent (Algorithm \ref{alg:simplified})
uses $\wt{O}_c(d^2/\eps^2)$ samples from
$p{(\matr W_*, \phi)}$,  performs $\wt O_c(d^4/\eps^6)$ gradient updates,
and converges to
a matrix $\wt{\matr W}$ where $\dtv(p{(\matr W_*, \phi)},$ $p{(\wt {\matr W}, \phi)} ) \leq \eps$
with probability $99\%$.
\end{theorem}

Theorem~\ref{infthm:main} shows that Nested Stochastic Gradient
Descent Ascent recovers
a parameter matrix $\wt{\matr{W}}$ such that the Generator distribution
$p(\wt{\matr W}, \phi)$ is close in total variation distance to the underlying distribution $p(\matr W_*, \phi)$. We note that as this is a gradient descent method the number of iterations naturally grows larger with the distance of the target matrix $\matr W_*$ to the initial weights $\matr W$ which are assumed to be $\matr I$\footnote{The dependence on the distance to $\matr I$ can be eliminated by an additional preconditioning step using a few samples from the target distribution (See Remark~\ref{remark:preconditioning}).}.

We also remark that the sample-complexity of our result is information theoretically optimal up to
polylogarithmic factors as even the case where the transformation $\phi$ is the identity transformation (i.e.,
$\phi(\bx) = \bx$) that corresponds to learning the covariance matrix of a Gaussian distribution is well known to require $\Omega(d^2/\eps^2)$
samples in order to learn (the covariance of) a Gaussian within total variation distance $\eps$.

A key challenge in showing Theorem~\ref{infthm:main} is to construct an appropriate Discriminator network that is powerful enough to distinguish between the true distribution and fake samples while, at the same time, simple enough to have few parameters and be efficiently trainable.

For any fixed Generator distribution there always exists a Discriminator that optimally distinguishes samples from the Generator and the target distribution (see Proposition 1
of \cite{Goodfellow}). Unfortunately, this Discriminator may be arbitrarily complex.
Even for a single layer neural network with ReLU activations, it requires treating samples differently according to their non-zero patterns, which is challenging to express directly as a simple low-depth neural network.

Instead, we focus on simple Discriminator networks that only discriminate samples that fall in the invertible region $T$ of the
transformation $\phi$.
In particular, our Discriminator first checks whether the received sample
belongs to (the image of) set $T$, then performs the inverse transformation
$\phi^{-1}$ followed by a quadratic layer and a sigmoid activation, see Figure~\ref{fig:dis_architecture}.
For the full GAN architecture see Figure~\ref{fig:full}.
We train both Generator and Discriminator with Nested Stochastic Gradient
Descent Ascent, that is we perform multiple iterations for the Discriminator per Generator
update.

Our choice of Discriminator allows us to use techniques and ideas from truncated
statistics, an area of statistics that deals with estimating the parameters of a distribution given only
conditional samples from a subset of the distribution. The Discriminator
essentially performs such a truncation operation to the data, see Figure~\ref{fig:dis_architecture}, as a result of the non-invertible function $\phi$.

Learning from high-dimensional truncated datasets is a notoriously challenging task and a computationally efficient algorithm for Gaussian data was only recently obtained
in \cite{daskalakis2018efficient} through maximum-likelihood estimation.
As a byproduct of our analysis, we show that the min-max GAN iteration is an alternative computationally efficient and near-sample optimal approach for the task of learning a truncated Gaussian considered in \cite{daskalakis2018efficient}.

 \subsection{Related Work}
Our work is inspired and motivated mainly by the success of Generative
Adversarial Neural nets in practice and aims to provide provable guarantees
for their convergence and sample complexity.  There are other works in the computer science and optimization communities that try to theoretically
analyze the behavior of GANs.  One such work related to ours is
\cite{feizi2017understanding}.  The authors consider the problem
of learning a Gaussian distribution using a Wassertstein GAN which corresponds to the special case of our model without a non-linearity, i.e., $\phi(\vec x) = \vec x$.
Another related work is \cite{gemp2018global}. The authors analyze a similar setting where the Generator is
again linear (learning a Gaussian distribution) and the Discriminator is
quadratic.  They train their W-GAN using a custom method that they denote as ``Crossing-the-curl''.
Interestingly, they show that simultaneous alternating SGDA diverges in their setting.
We view this as strong evidence that \emph{Nested} SGDA is indeed required
in order to have convergence for GANs. A third work in this direction is \cite{mescheder2018training}, where the authors study local convergence of different GAN architectures. In particular, the results show that GAN training diverges when the underlying distribution is not absolutely continuous.

Prior work also studied how well GANs generalize - does minimization of GAN's objective function offer any guarantee on the statistical distance between the Generator distribution and the target distribution? One such work is \cite{liu2017approximation}, where the authors address the problems by
giving a new notion of statistical distance (called \emph{adversarial divergence}) that captures a wide range of GAN objectives frequently used in practice. They show that for objectives falling in the category,
successfully optimizing the objectives implies weak convergence of the output distribution to the target.
Instead of treating it in a black-box manner, in our work we focus on the optimization process of specific GAN
instances and show that the output distribution converges to the target.

A more recent work related to ours is \cite{lei2019sgd}.
The authors prove that Wasserstein GANs can be trained
via SGD to learn one layer neural networks.  They assume
that the activation function has a separable form, i.e.,
$\phi(\vec x) = (q(\vec x_1), \ldots, q(\vec x_d))$, for some univariate function $q:\R \to \R$ that
has a simple form, e.g. is a Lipschitz and odd function.
Our result instead focuses on the standard GAN and shows that SGD converges for a much broader class of activation functions including non-invertible ones like ReLUs.

The interplay between min-max dynamics and GAN dynamics is already a very
active field of research.  An interesting recent work, that focuses on the
negative side of min-max games is \cite{flokas2019poincar}. The authors there
show that for a general class of non-convex, non-concave zero sum games
Stochastic Gradient Descent Ascent may not converge to fixed points that are meaningful within game theoretical settings.

On the positive side, in \cite{Arora2017ganeq}, the authors studied the
existence of pure equilibrium under various min-max game formulations of the
Generator/Discriminator training dynamics.  In \cite{rafique2018non}, a class of
non-convex concave optimization problem is studied, where the minimizer's
objective function is weakly convex and the maximizer's objective is strongly
concave.  Moreover, in \cite{lin2019gradient},  the performance of the Gradient
Ascent Descent (GDA) under similar setting is studied.  In \cite{NIPS2017},
Nagarajan et al. studied GAN's stability around the local Nash equilibrium of the
min-max game.  Finally, in \cite{DP18, DP19} the authors use optimism to show
convergence of gradient based methods in min-max optimization. \section{Preliminaries and Notation} \label{sec:notation}
We use small bold letters $\vec{x}$ to refer to real vectors in $\R^d$ and
capital bold letters $\matr{A}$ to refer to matrices in $\R^{d \times \ell}$.
 We define $\1\{\vec x \in S\}=\1_S(\vec x)$ to be the $0 - 1$ indicator of a set.
 The \emph{Frobenius norm} of a matrix $\matr{A}$ is defined as
$\norm{\matr{A}}_F = \sqrt{\sum_{ij} \matr A_{ij}^2}$. \\
For distributions $Q, P$ we denote by $\dtv(Q, P) = (1/2) \int |Q(x) - P(x)| d x$ their statistical
or total variation distance.
Let $\normal(\matr{\matr{W}})$ be the normal distribution with mean
$\vec 0 \in \R^d$ and covariance matrix $\matr{W} \matr{W}^T \in \R^{d \times d}$, with the following
probability density function
\begin{align} \label{eq:normalDensityFunction}
  \normal(\vec{x}; \matr{W}) =
  \frac{1}{\sqrt{\det \lp(2 \pi \wsqr \rp)}}
  \exp \left( - \frac{1}{2} \vec{x}^T \wsqr^{-1}
  \vec{x} \right).
\end{align}
Also, let $\normal(S; \matr W)$ denote the \emph{probability
mass of a measurable set $S$} under this Gaussian measure. We shall also denote
by $\normal$ the standard Gaussian; whether it is single or multidimensional
will be clear from the context. \\
In a multi-variable function, we often want to focus on just a subset of variables. We then use semi-colon to separate the primary variables from the secondary. Usually, the secondary variables are treated as constants which parametrize the function.
 \section{Technical Overview}
\subsection{Discriminator Design}
We study GAN Dynamics for learning one layer non-linear Neural Networks.  In particular we consider the following GAN architecture where the Generator is a one layer neural net with a fully connected linear layer parametrized by some matrix $\matr W \in \R^{d \times d}$ followed by some general non-linear activation function $\phi : \R^d \mapsto \R^d$.
Furthermore, we will use $\matr W_* \in \R^{d \times d}$ to denote the parameters of the target Generator network.
If we denote the density functions of Generator and the target distributions as $p_g$ and $p_d$ respectively, it is known \cite{Goodfellow} that
the optimal Discriminator for this problem is $D_*(\vec x; p_g) = \frac{p_d(\vec x)}{p_d(\vec x) + p_g(\vec x)}$.
Denote $\SIGMA = \matr W_* \matr W_*^T$. When the activation function $\phi$ is invertible over its  whole domain $\R^d$, this optimal Discriminator takes the following form
\begin{align*}
D_*(\vec x; \matr W) &= \frac{ \normal( \phi^{-1}(\vec x); \SIGMA^{1/2} ) } { \normal(\phi^{-1}(\vec x); \SIGMA^{1/2} ) + \normal( \phi^{-1}(\vec x); \matr W ) } \\
&= \sigma \lp( \phi^{-1}( \vec x)^T \matr A_* \phi^{-1}( \vec x) + b_* \rp) \, ,
\end{align*}
where $\matr A_* = \frac{1}{2} ( ( \matr W \matr W^T )^{-1} - \SIGMA^{-1})$, $b_* = \log \det (\matr W \SIGMA^{-1/2}) \, , $ $\sigma$ is the sigmoid function $\sigma(x) = \frac{1}{1 + \exp(-x)}$.

Unfortunately, many popular activation functions are not invertible over their whole domain but usually only on a subset of $\R^d$.  For example
the well-known ReLU activation i.e., $\text{ReLU}(\vec x) = \1 \{ \vec x \geq 0 \} \vec x$ is invertible only when every coordinate of $\vec x$
is positive.  In order to
capture these important activation functions, we relax the invertibility assumption to hold only on a subset of $\R^d$ (see Definition \ref{def:invertible_network}).
Recall that we denote by $p{(\matr W, \phi)}$ the output distribution of the network.
In this general setting, the optimal Discriminator may not have a simple form.
Even for ReLU (which is simply the identity function restricted on the set $T$)
the optimal Discriminator is a complicated piecewise function consisting of $2^d$ different cases (these cases correspond
to all possible subsets of coordinates that may be negative):
\begin{align*}
D(\vec x) &= \frac{\int_{ \vec z \in \mathcal Y_{\vec x}}  \normal(\vec z; \matr W_*) d \vec z}{\int_{ \vec z \in \mathcal Y_{\vec x}}  \normal(\vec z; \matr W) d \vec z + \int_{ \vec z \in \mathcal Y_{\vec x}}  \normal(\vec z; \matr W_*) d \vec z} \, , \\
& \text{where} \, \mathcal Y_{\vec x} := \{ \vec z \, \text{ such that } \, \vec x_i = \ReLU (\vec z_i) \}
\end{align*}
The optimal Discriminator is therefore a very complicated neural network and implementing such a
network is infeasible even in rather low-dimensional scenarios.
We take advantage of the fact that the activation function is known and invertible on some subset $T \in \R^d$
and design the following Discriminator architecture that balances simplicity and expressiveness.
We denote by $S$ the image of $T$ under $\phi$, i.e., $S = \phi(T)$ and define

\begin{figure} [H]
\centering
\begin{subfigure}{.5\textwidth}
  \centering
  \includegraphics[width=.5\textwidth]{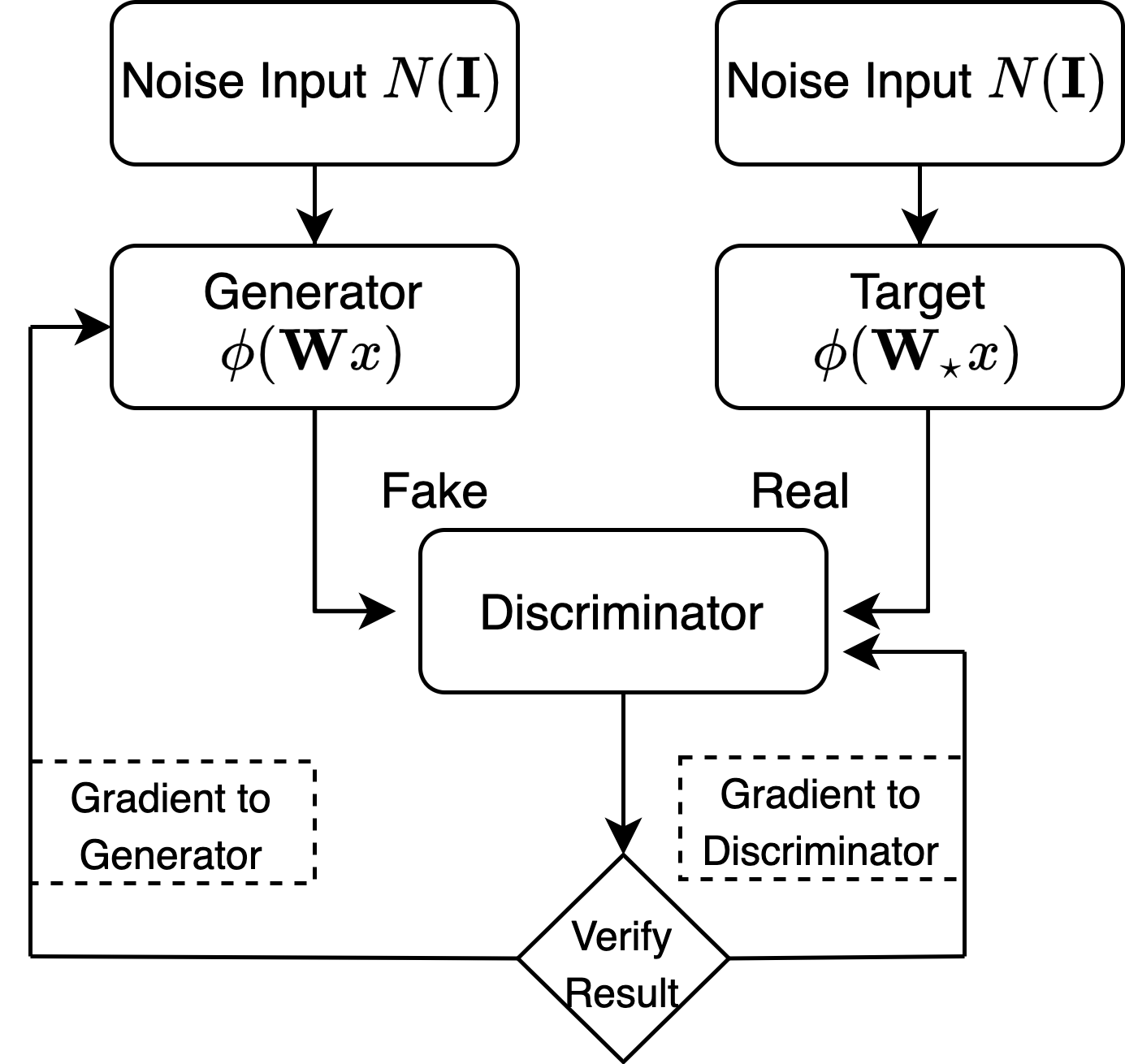}
  \caption{GAN Architecture Overview.} \label{fig:full}
  \label{fig:sub1}
\end{subfigure}\begin{subfigure}{.5\textwidth}
  \centering
  \includegraphics[width=.5\textwidth]{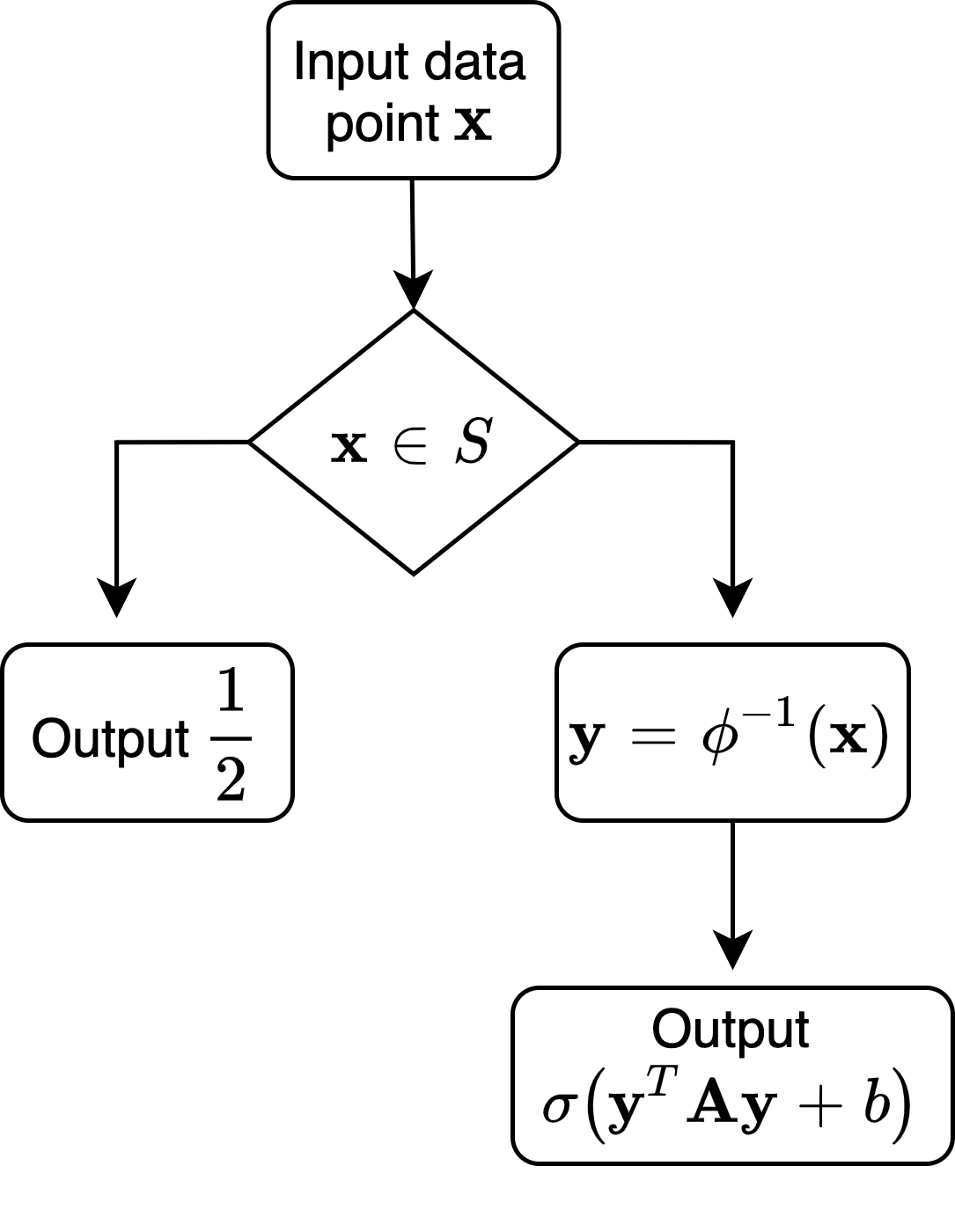}
  \caption{Discriminator Architecture. The set $S$ corresponds to the image of $T$ under $\phi$.} \label{fig:dis_architecture}
\end{subfigure}
\caption{Our GAN Architecture.}
\label{fig:test}
\end{figure}

\begin{align*}
D(\vec x; &\matr A, b) =
\1_S(\vec x) \sigma \lp( \phi^{-1}( \vec x)^T \matr A \phi^{-1}( \vec x) + b\rp)  \hspace{-.6mm} + \hspace{-.6mm}\frac{\1_{S^c}(\vec x)}{2} \, ,
\end{align*}
where $\sigma$ is the sigmoid function (see Figure ~\ref{fig:dis_architecture}). Then, the one layer Generator is paired with the above Discriminator to form the full architecture shown in Figure ~\ref{fig:full}. We then show that Nested Stochastic Gradient Descent Ascent (Algorithm ~\ref{alg:simplified}) on this pair of neural nets provably converges, enabling Generator to recover the target distribution.

\begin{algorithm} [H]
  \caption{Nested Stochastic Gradient Descent Ascent on Standard GAN }
  \label{alg:simplified}
 \hspace*{\algorithmicindent}
\begin{algorithmic}[1]
    \State Set $k = \wt{O}(d^2/\eps^2)$
    \State Initialize $\matr W = \matr I$.
    \State Sample $\vec x^{(1)}, \ldots, \vec x^{(k)}$ from $p{(\matr W_*, \phi)}$
\For{$i = 1$ to $\wt{O}(d^2/\eps^4)$}\vspace{10pt}

    \State  \(\triangleright\) { Discriminator Training }
    \State Randomly Initialize $\matr A$ and $b$.
    \State Sample $\vec y^{(1)}, \ldots, \vec y^{(k)}$ from $p{(\matr W, \phi)}$
    \For{$j = 1$ to $k$}
    \State Update $\matr A$ and $b$ with stochastic gradient
    \vspace{-2mm}
    \begin{align*}
        \nabla_{\matr A, b} \big[ \log( D(\vec x^{(j)};\matr A, b) ) + \log( 1 - D(\vec y^{(j)};\matr A, b) )  \big]
    \end{align*}
    \vspace{-6mm}
    \EndFor \vspace{10pt}
    \State  \(\triangleright\) { Generator Training }
\State Sample $\vec z \sim \normal(\matr I)$
    \State Update $\matr W$ with stochastic gradient
    \vspace{-2mm}
    \begin{align} \label{eq:abbreviated_generator_gradients}
        \nabla_{\matr W} \log( D(\phi(\matr W \vec z);\matr A, b)
    \end{align}
    \vspace{-6mm}
    \EndFor
  \end{algorithmic}
\end{algorithm}
We show that this Projected Nested Stochastic Gradient Descent-Ascent (NSGDA) algorithm converges.

\subsection{Roadmap of the Proof}
The zero sum game used in the GAN formulation corresponds to the min-max optimization problem with loss function
\begin{align} \label{eq:common_loss}
\mathcal{L}(\matr W, \matr A, b) =
&\E_{\vec x \sim p{(\matr W_*, \phi)} }
\log( D(\vec x;\matr A, b) )
+ \E_{\vec x \sim p{(\matr W, \phi)}}
\log( 1 - D(\vec x;\matr A, b) )
\end{align}
We use Nested Stochastic Gradient Descent-Ascent to solve the problem
$
\min_{\matr W} \max_{\matr A, b} \mathcal{L}(\matr W, \matr A, b).
$
The Nested SGDA solves this problem by trying to fully optimize the inner maximization optimization for a given Generator
parameter $\matr W$.
In other words, in the inner loop of Algorithm \ref{alg:simplified},
the Discriminator player is maximizing over $\vec A, b$ the objective function
$ \LD (\matr A, b ;\matr W) = \mathcal{L}(\matr W, \matr A, b)$; we stress that for $\LD$
the weight matrix $\matr W$ is a fixed parameter.

We first show that by doing Stochastic Gradient Ascend we can train Discriminator's parameters
 to being almost optimal. Using the structure of Discriminator we are able to show that $\LD$ is \emph{strongly concave} with respect to the
 Discriminator parameters $\matr A, b$, see Lemma \ref{lem:d-convexity}.  Since the Discriminator
 is using samples from the underlying model $p{(\matr W_*, \phi)}$, from a learning theoretic point of view, we want
 to make its optimization as efficient as possible in order to get tight sample complexity results.
 Strong concavity is crucial in that sense: we are able to depend optimally
 not only on the dimension $d$ but also on $\eps$; simple concavity would give us a substantially sub-optimal dependence on $\eps$.
The full discussion and detailed
versions of the corresponding lemmas can be found in Subsection~\ref{sub:training_discriminator}.

Showing convergence of Generator is more involved.  With Discriminator's parameters  $\matr A, b$ fixed,
in expectation, Generator in Algorithm \ref{alg:simplified} receives training gradients from the objective function
\begin{align}
    \LG (\matr W;\matr A, b) = \E_{\vec x \sim \normal(\matr I)}\log( 1 - D(\phi(\matr W \vec x);\matr A, b) ) \, .
\end{align}
By Danskin's Theorem \cite{danskin2012theory}, when $\matr A, b$  are fully optimized ($\matr A = \matr A_*, b= b_*$), the training gradients in expectation will be equal to the gradient of the function
$
\VV(\matr{W})
=
\max_{\matr A, b} \LD(\matr A, b; \matr W) \nonumber \, ,
$
which is known as the \emph{Virtual Training Criteria} of Generator in the work of \cite{Goodfellow}.
In contrast with the Discriminator objective function, minimizing $\VV(\matr {W})$ is a \emph{non-convex} minimization problem.
In fact, any factorization of the covariance matrix $\SIGMA = \bar{\matr W} \bar{\matr W}^T$ corresponds to
a minimizer of this problem: the Gaussian distribution is invariant under orthogonal transformations, and therefore
these matrices are  indeed indistinguishable since they all produce the same distribution.
Our main structural result shows that \emph{finding approximate stationary points} of the virtual training criteria $\VV(\matr {W})$
is sufficient to recover a matrix $\matr W$ whose corresponding distribution $p{(\matr W, \phi)}$
is close in total variation distance to the true underlying distribution  $p{(\matr W^{*}, \phi)}$.
The proof of this statement relies on Gaussian anti-concentration of polynomials, see Lemma~\ref{thm:cabrey}.
At a high level, we first argue that the norm of gradient of the Generator is proportional to the probability
that a specific quadratic form takes large values with respect to the standard normal distribution.
Then using anti-concentration we show that this probability cannot be too small unless
the distributions $p(\matr W, \phi)$ and $p(\matr W^*, \phi)$ are close.
For the formal statement of the above discussion see Lemma~\ref{lem:standard_optimality}.

A final complication that we face is that with finitely many samples,
it is impossible to recover the optimal Discriminator parameters $\matr A_*, b_*$ exactly.
This introduces \emph{biases} in the gradients used to train Generator.
To overcome the difficulty, we use a Biased PSGD lemma, which guarantees convergence of SGD to first-order stationary points of the underlying
objective function even when some bias are added to the gradient oracle used (See Lemma \ref{lem:biased_strongly_convex}).
In particular, the framework requires the bias to be bounded. We control the bias by showing that the  training gradients $\nabla_{\matr W} \LG(\matr W, \matr A, b)$ are Lipchitz continuous
with respect to the Discriminator parameters $\matr A, b$ (see Lemma \ref{lem:standard_cross_lip}).
Thus, as long as we train the Discriminator enough to ensure that $\matr A, b$ are close to the optimal $\matr A_*,  b_*$,
the bias $\snorm{}{\nabla_{\matr W} \LG(\matr W; \matr A, b) - \nabla_{\matr W} \VV(\matr W)}$ will be small.

\section{Convergence of GANs}
In this section, we prove our main result and show that the GAN iteration converges and learns the one-layer Generator network $(\matr W_*, \phi)$. Denote $\SIGMA = \matr W_* \matr W_*^T$. Without loss of generality, we can assume the underlying target network has the form $(\SIGMA^{1/2}, \phi)$ as
we have already seen that this does not affect the corresponding distribution.
The \emph{Generator} is a one-layer neural network of the form $(\matr W, \phi)$. The Generator will be paired with the
\emph{Discriminator} that tries to discern samples from $p{(\SIGMA^{1/2}}, \phi)$ and $p{(\matr W, \phi)}$.

If the Generator's parameter $\matr W$ is initialized very far from the target distribution, most of its samples will fall outside the truncation set $S$ of the Discriminator, leading to ``vanishing gradients". We thus make a closeness assumption that our initialization is close to the true covariance matrix. Assuming without loss of generality that we initialize the generator with $\matr W = \matr I$, we require that:

\begin{assumption}[Initialization] \label{initialization}
We assume that the matrix $\SIGMA$ satisfies
$$
\max(\|\SIGMA^{-1} - \matr I\|_F, \snorm{F}{\SIGMA - \matr I}) \leq c\,.
$$
\end{assumption}

\begin{remark}\label{remark:preconditioning}
As shown in Corollary 3 of \cite{Daskalakis2018EfficientSI}, we can initialize the algorithm with the
empirical covariance matrix computed using $O_{\alpha}(d^2)$ samples from the truncated normal distribution
$\normal(\SIGMA^{1/2}, T)$ and then transform the space so that $\matr W^{(0)} \rightarrow \matr I$.
Then constant $c$ in Assumption \ref{initialization} depends only on the mass of  the set $\alpha$, i.e.,
$c = \poly(1/\alpha) = O(1)$ under our assumption that $\alpha = \Omega(1)$.
\end{remark}

\subsection{Projection Set} \label{sec:projection_set}

In order to avoid moving towards regions where the gradients vanish
we will use the following convex projection set
for the Generator parameters $\matr W$.
\begin{align} \label{eq:proj_g}
& \mathcal{Q}_G =\bigg\{
\snorm{F}{\matr{W} - \matr I} \leq \poly(c)\, ,\poly(1/c) \leq  \vec x^T \matr{W} \vec x \leq \poly(c), \text{ for all } \snorm{2}{\vec x} = 1
 \bigg\} \, ,
\end{align}
The important property of the above projection set is
that the set $T$ (the set where $\phi$ is invertible)  has non-trivial
mass under any matrix $\matr W \in \mcal{Q}_{G}$.
Interpreting the set $T$ as a truncation set and using tools developed in  \cite{Daskalakis2018EfficientSI},
we can show that the set $T$
always has non-trivial mass with respect to the Gaussian distribution $\normal(\matr W)$.
This is a crucial property because our Discriminator relies on seeing samples that
fall inside the set $S$ (recall that $S$ is the image of $T$ under $\phi$).  In
order for the Discriminator to produce non-trivial gradients we need to ensure
that the mass of the invertible set $T$ is not-trivial with respect to the
parameter of the Generator.
\begin{lemma}[Non-trivial mass] \label{lem:non_trivial_mass}
Under Assumption \ref{initialization},
if we have $\matr W \in \mathcal Q_G$, it holds that $\normal(T;\matr W) =\Omega_{c}(1)$.
\end{lemma}
As we discussed previously, to obtain the optimal sample complexity we require the loss function
of the Discriminator to be strongly concave. Unfortunately, strong concavity does not hold globally for the objective function $\mathcal L$.
Hence, we  shall define the following projection set for Discriminator's parameters that ensures (see Lemma~\ref{lem:d-convexity}) this desired property.
\begin{align} \label{eq:proj_d}
\mathcal{Q}_D = \bigg\{ &  \snorm{F}{ \matr A  } \leq \poly(c) ,  \abs{b} \leq \poly(c) \bigg\}
\,.
\end{align}
We remark both sets $\mathcal Q_G$ and $\mathcal Q_D$ are convex and their projections can be efficiently computed,
see, for example, Algorithm 3 in \cite{Daskalakis2018EfficientSI}.

\subsection{Training the Discriminator}
\label{sub:training_discriminator}
In this section, we show convergence property of Discriminator training given in the following proposition.
\begin{proposition} [Convergence of Discriminator Training] \label{prop:discriminator_convergence}
Fix $\matr W \in \mathcal Q_G$ and assume that Assumption \ref{initialization} is true.
Setting the inner loop for $M_{\D} = \wt O(d^2/ \eps^2 \log^2(1/\delta))$. Then, with probability at least $1 - \delta$, it holds when Algorithm \ref{alg:simplified} exits the inner loop, the parameters $\matr A, b$ satisfy
$
\snorm{F}{\matr A - \matr A_*} + \abs{b - b_*} \leq \eps \, ,
$
where $\matr A_* = \frac{1}{2} ( ( \matr W \matr W^T )^{-1} - \SIGMA^{-1} ), b_* = \log \det (\matr W \SIGMA^{-1/2})$.
\end{proposition}

The main step of the proof is the following Lemma which shows that the loss function $\mathcal L$ as specified in Equation \eqref{eq:common_loss} is strongly concave with respect to $\matr A, b$. Its proof relies on the anti-concentration of polynomials under the Gaussian measure (Lemma \ref{thm:cabrey}) and can be found in the Appendix  \ref{sec:discriminator_convexity}.
\begin{lemma} [Strong concavity for Discriminator] \label{lem:d-convexity}
Fix the Target Network $(\matr W_*, \phi)$ and Generator Network $(\matr W, \phi)$,
it holds that $\LD(\matr{A},b; \matr{W})$ is at least $\Omega_{c}(1)$-strongly concave
when $\matr{A}, b \in \mathcal{Q}_D$ described in Equation \eqref{eq:proj_d}, $\matr W \in \mathcal Q_G$ described in Equation \eqref{eq:proj_g} and Assumption \ref{initialization} is satisfied.
\end{lemma}

\subsection{Training the Generator}
As we discussed previously, the Generator tries to optimize the loss function
\begin{align} \label{eq:standard_generator_objective}
\LG(\matr{W};\matr{A}, b) &=
\E_{\vec x \sim \normal} \log
\lp( 1 - D(\phi(\matr{W}\vec{x}); \matr{A}, b)\rp)
\end{align}
where $D(\vec{x}; \vec A, b) = \1_S(\vec x)/\lp( 1 + \exp \lp( -\vec{x}^T \matr{A} \vec{x}  - b\rp)
\rp) + \1_{S^c}(\vec x)/2$.
By Danskin's Theorem \cite{danskin2012theory}, when we use the optimal Discriminator parameters, namely $\matr A_* = \frac{1}{2} \lp( \lp( \matr
W \matr W^T \rp)^{-1} - \SIGMA^{-1} \rp), b_* = \log \det \matr W - \log \det \SIGMA^{1/2}$,
we essentially optimize the function
\begin{align}
\VV(\matr{W})
&=
\max_{\matr A, b} \LD(\matr A, b; \matr W) \nonumber \\
&= \E_{\vec x \sim \normal(\matr W) } \lp[ \log(1 - D(\phi(\vec x); \matr A_*, b_*) ) \rp]  + \E_{\vec x \sim \normal(\SIGMA^{1/2})} \lp[ \log( D( \phi(\vec x); \matr A_*, b_*) ) \rp].
\end{align}
When the Discriminator is not fully optimized, the training gradients can still be treated as \emph{biased} estimators of the true gradients of
$\VV (\matr W)$.  We first ignore the bias introduced from the sub-optimal Discriminator
and prove our main structural result,
showing that finding any stationary point of the Virtual Training Criteria $\VV(\matr W)$ suffices to learn the
underlying distribution.
Since $\VV(\cdot)$ is not convex and we have a projection set, there are many obstacles in optimizing this objective function.
Firstly, we need to make sure that stationary points in the interior of $\mathcal Q_G$ are close to being optimal.
Secondly, we need to make sure that the projection set does not introduce new ``bad" stationary points (that is matrices $\matr W$
whose corresponding distribution $p(\matr W, \phi)$ is far from $p(\SIGMA, \phi)$.
lying on the boundary.
To do so, we will employ the anti-concentration property of polynomials under Gaussian measure, which is stated in the following lemma.

\begin{lemma}[Theorem 8 of \cite{CarberyW01}] \label{thm:cabrey}
  Let $k, \gamma \in \R+$, $\vec{m} \in \R^d$, $\matr{\Sigma} \in \R^{d
  \times d}$ such that $\matr \Sigma$ is positive semidefinite and $p
  : \R^d \to \R$ be a multivariate polynomial of degree at most $\ell$, we
  define
  $ \bar{Q} = \left\{ \vec{x} \in \R^d \mid \abs{p(\vec{x})} \le \gamma
  \right\}, $
  then there exists an absolute constant $C$ such that
  \[ \normal(\bar{Q}; \matr{\Sigma}^{1/2})
  \le \frac{C k \gamma^{1/\ell}}
    {\left( \E_{\vec{z} \sim \normal(\matr{\Sigma}^{1/2})}
  \left[ \abs{p(\vec{z})}^{k/\ell} \right] \right)^{1/k}}. \]
\end{lemma}

We are now ready to show the optimality of stationary points with respect to the learning problem.
\begin{definition} \label{def:stationary_point}
A point $\vec w \in \mathcal Q$ is an $\eps$-approximate first order stationary point of the function $f: \R^d \mapsto \R$ ($\eps$-FOSP) if
and only if for all $\vec u \in \mathcal Q$ the following holds
\begin{align*}
    \frac{1}{\snorm{2}{\vec w - \vec u}} \langle \nabla_{\vec w}f(\vec w), \vec w - \vec u \rangle \leq \eps.
\end{align*}
\end{definition}

\begin{lemma}[Stationary Points Suffice] \label{lem:standard_optimality}
Let $\matr W$ be an $\eps$-first order stationary point ($\eps$-FOSP)
of $\VV(\matr W)$ in $\mathcal Q_G$. Then it holds
$
 \dtv \lp( p{(\matr{W}, \phi)}, p{(\matr W_*, \phi)} \rp) \leq O_{c}(\eps) \,.
$
\end{lemma}

\begin{proof}[Proof Sketch.]
Here we only deal with the case when $\matr W$ is an interior point of $\mathcal Q_G$. The rest of the proof which considers the case when $\matr W$ lies on the boundary of $\mathcal Q_G$ and can be found in Appendix \ref{sec:standard_optimality}.
For convenience, we define the expressions
\begin{align}
&h(\vec{x};\matr{W}) = \frac{1}{2} \vec{x}^T \lp( \lp( \matr{W} \matr{W}^T \rp)^{-1} - \SIGMA^{-1} \rp) \vec{x}  + \log \det \matr{W} \SIGMA^{-1/2} \\
& f(y) = \log(1 + \exp(y))
\end{align}
Then, the gradient of the Virtual Training Criteria $\VV(\matr W)$ is given by
\begin{align}
\nabla_{\matr{W}} &\VV(\matr{W})
= \lp( \SIGMA^{-1} \matr{W} - \lp( \matr{W}^{-1} \rp)^{T} \rp) \cdot \E_{\vec{x} \sim \normal} \lp[ f^{'}(h(\matr{W}\vec{x};\matr{W})) \XX \1\{ \matr W \matr x \in T\} \rp]
\end{align}
Given two matrices $\matr A, \matr B \in \R^{d \times d}$ where $\matr B$ is a symmetric positive definite matrix, we always have $\snorm{F}{\matr A \ \matr B} \geq \snorm{F}{\matr A} \min_{ \snorm{2}{\vec z} = 1} \vec z^T \matr B \vec z$.
Hence, the frobenius norm of the gradient can be lower bounded by
\begin{align*}
& \snorm{F}{\nabla_{\matr{W}} \VV(\matr{W})}
\geq
\snorm{F}{\lp( \SIGMA^{-1} \matr{W} - \lp( \matr{W}^{-1} \rp)^{T} \rp)} \cdot \min_{ \snorm{2}{\vec{z}} = 1 }\E_{\vec{x} \sim \normal} \lp[ \1 \{ \matr W \matr x \in T\} f^{'}(h(\matr{W}\vec{x};\matr{W})) \lp(\vec z^T \vec x\rp)^2  \rp]
\end{align*}
We now bound from below $\min_{ \snorm{2}{\vec{z}} = 1 }\E_{\vec{x} \sim \normal} \lp[ \1 \{\matr W \matr x  \in T\} f^{'}(h(\matr{W}\vec{x};\matr{W})) \lp(\vec z^T \vec x\rp)^2   \rp]$.
Notice that $f^{'}(y) = \sigma(y) = 1/(\exp(-y)+1)$ is the sigmoid function.
Using the property that $f^{'}(\cdot)$ is positive, and non-decreasing and also that $\XX$ is
positive semi-definite, we get the following inequality
\begin{align*}
& \E_{\vec{x} \sim \normal} \lp[ \1 \{\matr W \matr x  \in T\} f^{'}(h(\matr{W}\vec{x};\matr{W})) \lp(\vec z^T \vec x\rp)^2  \rp] \\
& \geq f^{'}(r) \gamma   \E_{\vec{x} \sim \normal} [ \1\{ h(\matr{W}\vec{x};\matr{W}) \geq r \}  \1\{ \lp(\vec z^T \vec x\rp)^2 \geq \gamma\}
\cdot
\1 \{\matr W \matr x  \in T\} ]
\,.
\end{align*}
Since we know that $\matr W \in \mathcal Q_G $,  we can prove (see Appendix for details)
that   $ \E_{\vec{x} \sim \normal} \abs{h(\matr{W}\vec{x};\matr{W})} \leq \poly(c)$.
Thus, if we choose $r<0$, by Markov's inequality, we have
\begin{align*}
& \normal \lp( h(\matr{W}\vec{x};\matr{W}) \leq r \rp)
\leq
\normal \lp( \abs{h(\matr{W}\vec{x};\matr{W})} \geq \abs{r} \rp) \leq \frac{1}{\abs{r}} \E_{\vec{x} \sim \normal} \abs{h(\matr{W}\vec{x};\matr{W})}
\end{align*}
By Lemma \ref{lem:non_trivial_mass}, the mass $\normal(T;\matr W)$ is always lower bounded by some absolute constant $k_{c}$ that depends only on $c$. Hence, by setting $r= -\frac{4}{k_{c}} \E_{\vec{x} \sim \normal} \abs{h(\matr{W}\vec{x};\matr{W})}$, we have $\normal \lp( h(\matr W\vec{x};\matr W) \leq r \rp) \leq k_{c}/4$.
On the other hand, we have $\E_{\vec{x} \sim \normal} \lp[ \lp(\vec z^T \vec x\rp)^2 \rp] = 1$ given
that $\snorm{2}{\vec{z}}=1$.

Now we can use the Gaussian anti-concentration of polynomials,
Lemma~\ref{thm:cabrey}, for the degree $2$ polynomial $\lp(\vec z^T \vec x\rp)^2$.
We choose
$$\gamma = \frac{1}{2} \lp( \frac{k_{c}}{8C} \rp)^2 \E_{\vec{x} \sim \normal} \lp[ \lp(\vec z^T \vec x\rp)^2 \rp] \, ,$$
and therefore, we have $\normal\lp( \lp(\vec z^T \vec x\rp)^2 \leq \gamma \rp) \leq k_{c}/4$.
Thus, by union bound, we conclude
\begin{align*}
&\E_{\vec{x} \sim \normal}\lp[ \1\{ h(\matr{W}\vec{x};\matr{W}) \geq r \} \1\{ \lp(\vec z^T \vec x\rp)^2 \geq \gamma \} \1\{ \matr W \matr{x} \in T \} \rp] \\
&\geq k_{c} - k_{c}/4 - k_{c}4
\geq k_{c}/2.
\end{align*}
Using the inequality $f^{'}(y) \geq e^y/2$ when $y<0$, we obtain the bound
\begin{align*}
\min_{ \snorm{2}{\vec{z}} = 1 }\E_{\vec{x} \sim \normal} \lp[ f^{'}(h(\matr{W}\vec{x};\matr{W})) \lp(\vec z^T \vec x\rp)^2 \rp]
\geq \Omega_{c}(1)
\end{align*}
Therefore, given $\snorm{F}{\nabla_{\matr{W}} \VV(\matr{W})} \leq \eps$, it holds
\begin{align*}
& \snorm{F}{ \SIGMA^{-1/2} \lp( \matr{W} \matr{W}^T  \rp) \SIGMA^{-1/2} - I  }
\leq \snorm{F}{\lp( \SIGMA^{-1} \matr{W} - \lp( \matr{W}^{-1} \rp)^{T} \rp)} c^2
\leq O_{c}(\eps)
\end{align*}
Using Pinsker's inequality (and the exact expression of Kullback-Leibler divergence for normal distributions) we have
\begin{align*}
\dtv &(\normal( {\matr W}), \normal(\SIGMA^{1/2}) \leq
\snorm{F}{ \SIGMA^{-1/2}  \lp(  {\matr W}  {\matr W}^T  \rp) \SIGMA^{-1/2} - \matr I}
\leq
O_c(\eps)\,.
\end{align*}
Using the data processing inequality it follows that the total variation distance between
the transformed distributions $p{(\matr{W}, \phi)}$, $p{(\matr W_*, \phi)}$
is small, i.e.,
$
 \dtv \lp( p{(\matr{W}, \phi)}, p{(\matr W_*, \phi)} \rp) \leq O_{c}(\eps) \,.
$
\end{proof}

We have seen that finding stationary points of the non-convex objective suffices in order
to compute a good parameter matrix $\matr W$.  However, as we have already discussed
we cannot optimize the Discriminator exactly and this leads to biased gradients when
we train the Generator, that is gradients that do not exactly match the stochastic
gradients of the function $\VV(\matr W)$.  We now show how to overcome this obstacle.
If we compute the gradient given in Equation \eqref{eq:abbreviated_generator_gradients}, we get
\begin{align} \label{eq:expanded_generator_gradients}
    &\vec g_{\GEN} = D \left( \phi(\matr{W} \vec{z}) ;\matr{A}, b \right)\matr{A} \matr{W} \vec{z}\vec{z}^T \, ,
\end{align}
where $\vec z \sim \normal(\matr I)$.  In the following lemma, we show that the bias can be controlled as long as the
parameters of the Discriminator are approximately optimal.  In particular, we prove that
the gradients $\nabla_{\matr W} \LG(\matr W)$ (namely the gradient oracle $g_{\GEN}$ in expectation) are Lipschitz with respect to Discriminator's parameters.
\begin{lemma} \label{lem:standard_cross_lip}
   $\E_{\vec z \sim \normal(\matr I)} \lp[ \vec g_{\GEN} \rp] = \nabla_{\matr{W}}\LG(\matr{W};\matr{A}, b)$ is $O_c(1)$-Lipchitz with respect to $\matr{A}$ and $b$ when $\matr{W} \in \mathcal{Q}_G$  $\matr{A}, b \in \mathcal{Q}_D$ and Assumption \ref{initialization} is satisfied.
\end{lemma}

Apart from that, we also need that the variance of the gradient oracle is bounded.
We show the following lemma (see Appendix \ref{sec:standard_bounded_variance}).
\begin{lemma} \label{lem:standard_bounded_variance}
Let $
    \vec g_{\GEN} = D \left( \phi(\matr{W} \vec{z}) ;\matr{A}, b \right)\matr{A} \matr{W} \vec{z}\vec{z}^T  $
    be the gradient update of the Generator network
    and assume that  $\matr{A}, b \in \mathcal{Q}_D$ described in Equation \eqref{eq:proj_d},
    $\matr W \in \mathcal Q_G$ described in Equation \eqref{eq:proj_g}, and that Assumption \ref{initialization} is satisfied.
    Then it holds that $\E_{\vec z \sim \normal(\matr I)} \lp[ \snorm{2}{\vec g_{\GEN}}^2 \rp] \leq O_c(d^2)$
\end{lemma}

Finally, we prove the Biased SGD Lemma which shows that the properties guaranteed by Lemmas ~\ref{lem:standard_cross_lip} and  ~\ref{lem:standard_bounded_variance} are essentially enough for us to optimize the Virtual Training Criteria.
Technically,
its proof is similar to the work of \cite{ghadimi2016mini}; we
adapt it so that it handle biased gradients (see Appendix \ref{sec:biased_strongly_convex}).
\begin{lemma}[Biased Nonconvex PSGD] \label{lem:biased_strongly_convex}
  Let $f$ be an $l$-Lipschitz and $L$-smooth function, such that
  $\max_{\vec x, \vec y \in \mathcal{Q}} \snorm{2}{f(\vec x) - f(\vec y)}
  \leq R$ on a convex domain $\mathcal{Q}$.
  At step $t$ of the SGD we are given a biased gradient $\vec \xi^{(t)}$ such that $\snorm{2}{\E[\vec \xi^{(t)}| \vec
  \xi^{(1)},\ldots \vec \xi^{(t-1)}] - \nabla f(\vec x)} \leq \alpha$
  and $\E[\snorm{2}{\vec \xi^{(t)}}^2] \leq B$.
  Set $M = O(BLR/\eps^4)$ and sample the stopping time $m$ uniformly at random from $\{1, \cdots, M \}$.
  Then, with step size $\beta = \sqrt{2R/(LBM)}$ and the update rule $\vec w^{(t+1)} = \argmin_{\vec w \in \mathcal Q} \snorm{2}{ \vec w - (\vec w^{(t)} - \beta \vec \xi^{(t)}) }$,
  we have that with probability at least
  $99\%$, the last iteration $\vec w^{(m)}$ of PSGD is an $O(\eps + \sqrt{l \alpha})$-stationary point
 of $f$.
\end{lemma}

Finally, we will combine the lemmas together with the Biased-SGD framework to obtain the sample complexity and number of iterations needed of Algorithm \ref{alg:simplified}.
\subsubsection{Proof of Theorem \ref{infthm:main}}

Using Proposition \ref{prop:discriminator_convergence},
if we run the inner loop for $M_{\D} = O_c(d^2/ \eps^2 \log^2(1/\delta))$ iterations, with probability at least $1 - \delta$, when
Algorithm~\ref{alg:simplified} exits the inner loop, the parameters $\matr A, b$ satisfy
$
\snorm{F}{\matr A - \matr A_*} + \abs{b - b_*} \leq O_c(\eps)
$.
Using Lemma \ref{lem:standard_cross_lip}, we know $\nabla_{\matr W} \LG(\matr{W}, \matr{A}, b)$ is $\tau : = \Omega_c(1)$-Lipchitz with respect to Discriminator's parameters $\matr{A},b$. Furthermore, $\nabla_{\matr W} \LG(\matr{W}, \matr{A}_*, b_*) = \nabla_{\matr W} \VV(\matr W)$. Hence, it holds
\begin{align*}
& \snorm{F}{\nabla_{\matr W} \LG(\matr W; \matr A, b) - \nabla_{\matr W} \VV(\matr W)} \leq \tau \lp( \snorm{F}{\matr A - \matr A_*} + \abs{b - b_*} \rp) \leq O_c(\eps).
\end{align*}
This implies that the gradient oracle  $\vec g_{\GEN}$ used by Algorithm \ref{alg:simplified} satisfies
\begin{align} \label{eq:inner_gurantee}
\snorm{F}{\E_{\vec x \sim \normal} \lp[ \vec g_{\GEN}(\vec x; \matr W, \matr A, b) \rp] - \nabla_{\matr W} \VV(\matr W)} \leq O_c(\eps).
\end{align}
We have that the virtual objective function $\VV(\matr{W})$ is
$L := O_c(1)$-smooth and $l := O_c(1)$-Lipchitz continuous (see Appendix for a proof).
Moreover, using Lemma~\ref{lem:standard_bounded_variance}, we have that
the variance of the gradient oracle, namely $\E_{\vec x \sim \normal} \lp[ \snorm{}{ \vec g_{\GEN} (\vec x) }^2 \rp]$, is
bounded by $B = O_c(d^2)$. By the definition of the Projection Set $\mathcal Q_G$ in Equation \eqref{eq:proj_g} and the fact that $\VV(\matr W)$ is $O_c(1)$-Lipchitz continuous, it holds
\begin{align*}
R := &\max_{\matr W_1, \matr W_2 \in \mathcal Q_G} \abs{ \VV(\matr W_1) - \VV(\matr W_2) }
\leq O_c(1).
\end{align*}
Conditioning on the event that the guarantee in Equation \eqref{eq:inner_gurantee} is met, by Lemma \ref{lem:biased_strongly_convex}, if we run
the outer loop of Algorithm~\ref{alg:simplified} for $M_{\mathcal{G}} = O_c(L B R /\eps^4) = \wt O_c(d^2/ \eps ^4)$
rounds with step size $\eta_{\GEN} = \sqrt{ \frac{L B}{ R
M_{\mathcal{G}}} } = O_c(\eps^2)$, it holds that the last iteration Generator parameters $\wt {\matr W}$ are an $O_c(\eps)$-
first order stationary point of of $\VV (\matr W)$. If we set $\delta = \frac{1}{100 M_{\mathcal G}}$, by the union bound, the probability that Equation
\eqref{eq:inner_gurantee} fails to hold in any iteration is less than $1\%$. Finally, by Lemma \ref{lem:standard_optimality}, we
can transform
the guarantee into bounds on the total variation distance between Generator distribution and target distribution and conclude that with probability at least $99 \%$ in the last iteration
$\dtv (p(\wt {\matr W}, \phi), p{(\SIGMA^{1/2}, \phi)} ) \leq \eps$.
\newpage
\bibliographystyle{alpha}
\bibliography{refs}

\appendix
\section{Additional Notation}
In this section we define some additional notation used in the following
sections of the appendix.
We denote
by $\matr A \otimes \matr B$ the Kronecker product between two matrices
$\matr A \in \R^{m \times n}$,
$\matr B \in \R^{k \times \ell}$,
is the block matrix
$$
\matr A \otimes \matr B =
\begin{bmatrix}
  a_{11} \mathbf{B} & \cdots & a_{1n}\mathbf{B} \\
             \vdots & \ddots &           \vdots \\
  a_{m1} \mathbf{B} & \cdots & a_{mn} \mathbf{B}
\end{bmatrix}
\in \R^{m k \times n \ell }
$$
We also define the
symmetrization of a matrix $\mcal{S}(\matr A) = \matr A + \matr A^{T}$.
 \section{Projection Set}
Recall that the projection sets for Generator and Discriminator are given by
\begin{align*}
& \mathcal{Q}_G =\bigg\{
\snorm{F}{\matr{W} - \matr I} \leq \poly(c)\, ,\poly(1/c) \leq  \vec x^T \matr{W} \vec x \leq \poly(c), \text{ for all } \snorm{2}{\vec x} = 1
 \bigg\} \, , \\
& \mathcal{Q}_D = \bigg\{ \snorm{F}{ \matr A  } \leq \poly(c) ,  \abs{b} \leq \poly(c) \bigg\}
\,.
\end{align*}

The following lemma shows that the above projection sets always contain valid solutions of
the problem.
Morerover, when the parameters of the Discriminator and the Generator lie inside their corresponding
projection sets we have the following bounds that will be useful throughout our analysis.
\begin{lemma}[Projection Sets] \label{lem:generator_projection_property}
Under Assumption \ref{initialization},
 we have that the convex set $\mathcal{Q}_G$ contains some matrix $\matr W$ such that
  the corresponding distribution $p{(\matr W, \phi)}$ is equal to
  the true underlying distribution $p{(\matr W_*, \phi)}$.
  Moreover, for any $\matr W \in \mathcal{Q}_G$ we have that
  the optimal parameters for the Discriminator,
  $
\matr A_* = \frac{1}{2} \lp( \lp( \matr W \matr W^T  \rp)^{-1} - \SIGMA^{-1} \rp) , b_* = \log \det \lp( \matr W \SIGMA^{-1/2} \rp)
$
lie in the Discriminator projection set $\mathcal{Q}_D$.
  Finally,  for all $\matr W \in \mathcal Q_G$ we have that the following bounds hold
 $
 \snorm{F}{\matr{W}\matr{W}^T - \SIGMA},
 \snorm{F}{\lp( \matr{W}\matr{W}^T \rp)^{-1} - \SIGMA^{-1}},
 \snorm{F}{ \matr{W}^T \SIGMA^{-1} \matr{W} - \matr{I} },$
 $
 \snorm{F}{\SIGMA^{-1/2} \matr W  \matr W^T \SIGMA^{-1/2} - \matr I}$,\\
 $\snorm{F}{\SIGMA^{1/2} \lp( \matr W \matr W^T \rp)^{-1} \SIGMA^{1/2} - \matr I} ,$
 $
 \abs{\log \det \lp(  \matr W \SIGMA^{-1/2} \rp)} \leq \poly(c)$
\end{lemma}
\begin{proof}
First, we consider the the projection set $\mathcal Q_G$. By Assumption \ref{initialization}, we have $ \snorm{F}{ \SIGMA - \matr{I} } \leq c$.
Assume that we have the following eigenvalue decomposition $\SIGMA = \matr{U} \matr{\Lambda}^2 \matr{U}^T$.
Then the inequality assumed can be rewritten as $\snorm{F}{ \matr{U} \lp( \matr{\Lambda}^2 - \matr{I} \rp) \matr{U}^T } \leq c$. Since the Frobenious norm is invariant under unitary transformations, this gives $\snorm{F}{\matr{\Lambda}^2 - \matr{I}} = \snorm{F}{  \lp( \matr{\Lambda} - \matr{I} \rp) \lp(  \matr{\Lambda} + \matr{I} \rp) } \leq c$. As $\matr{\Lambda} + \matr{I}$ clearly has its eigenvalues lower bounded by $1$,
it implies
$
\snorm{F}{ \SIGMA^{1/2}  - \matr I }
=
\snorm{F}{\matr{U} \lp( \matr{\Lambda} - \matr{I}  \rp) \matr{U}^T}
= \snorm{F}{ \Lambda  - \matr{I}}
\leq
c.
$
On the other hand, since $\snorm{2}{\SIGMA^{1/2} -I} \leq \snorm{F}{\SIGMA^{1/2} - I} \leq c$, it holds $\snorm{2}{\SIGMA^{1/2}} \leq 1 + c$. Similarly, $\snorm{2}{\SIGMA^{-1/2}} \leq 1 + c$. Thus, the eigenvalues of $\SIGMA^{1/2}$ lie in the interval $[1/(1+c), 1+c]$.
Hence, we have shown that the projection set $\mathcal Q_G$ contains some matrix $\matr W_*$ that is essentially optimal
(up to orthogonal transformations).
On the other hand, the six expressions in the statement are all $\poly(c)$-Lipchitz with respect to $\matr W$ as the l2-norms of $\matr W$, $\matr W^{-1}$, $\SIGMA^{1/2}$, $\SIGMA^{-1/2}$ are all bounded by $\poly(c)$. The upper bounds of these expressions then follow from their Lipchitzness, the diameters of the projection set ($\poly(c)$) and the fact that they all evaluate to $0$ when $\matr W = \SIGMA^{1/2}$.

Next, we consider the Discriminator projection set. Recall that after fixing the Generator parameters $\matr W$, the optimal Discriminator parameters are given by
$
\matr A_* = \frac{1}{2} \lp( \lp( \matr W \matr W^T  \rp)^{-1} - \SIGMA^{-1} \rp), b_* = \log \det \lp( \matr W \SIGMA^{-1/2} \rp)
$
. From our discussion of Generator Projection Set, we know both expressions are bounded by $\poly(c)$ when $\matr W \in \mathcal Q_G$. Hence, for any $\matr W \in \mathcal{Q}_G$ we have that
the corresponding optimal Discriminator parameters lie in the projection set $\mathcal Q_D$.

\end{proof}
\subsection{Proof of Lemma \ref{lem:non_trivial_mass}}
We will use the following lemma from the work of \cite{Daskalakis2018EfficientSI} which relates the probability mass that two different normal distributions assign to the same set.
\begin{lemma}[Lemma 7 of \cite{Daskalakis2018EfficientSI}] \label{lem:normal_mass}
Consider two normal distributions $\normal(\matr \Sigma_1^{1/2}), \normal(\matr \Sigma_2^{1/2})$ and a set $S$
satisfying $\normal(S;\matr \Sigma_1^{1/2}) \geq \alpha$.
Suppose the parameters satisfy $\snorm{F}{ \matr \Sigma_1^{1/2} \matr \Sigma_2^{-1} \matr \Sigma_1^{1/2} - \matr I} \leq B$. Then, it holds $\normal(S;\matr \Sigma_2^{1/2}) \geq k_{B, \alpha}$ for some constant $k_{B, \alpha}$ that depends only on $B$ and $\alpha$.
\end{lemma}
From Lemma \ref{lem:generator_projection_property}, we know that $\snorm{F}{\SIGMA^{1/2} \lp( \matr W \matr W^T \rp)^{-1} \SIGMA^{1/2} - \matr I} \leq \poly(c)$.
Recall that by Definition \ref{def:invertible_network}, we have $\normal(T, \SIGMA^{1/2}) \geq \alpha = \Omega(1)$.
From Lemma \ref{lem:normal_mass} it follows that $\normal(T;\matr W) \geq \Omega_c(1)$.

\section{Training the Discriminator} \label{sec:appendix_dis}
For convenience, we define the following expressions which commonly appear in the formulas of the training gradients.
Let $
h(\vec{x};\matr{A}, b) = \vec{x}^T \matr{A} \vec{x} + b,h(\vec{x};\matr{W}) = \frac{1}{2} \vec{x}^T \lp( \lp( \matr{W} \matr{W}^T \rp)^{-1} - \SIGMA^{-1} \rp) \vec{x} + \log \det \matr{W} \SIGMA^{-1/2}\, ,\sigma(y) = 1/(1 + \exp(-y)) \, ,
f(y) = \log(1+\exp(y))
$. Notice that the second expression is equivalent to the first expression when $\matr A$, $b$ are exactly the optimal Discriminator parameters.
\subsection{Proof of Lemma \ref{lem:d-convexity}} \label{sec:discriminator_convexity}

We will use the following facts.
\begin{fact} \label{fact:sym_decompose}
For any symmetric matrix $\matr{X} \in \R^{d \times d}$, there exist two semidefinite matrices $\matr{Y}$ and $\matr{Z}$ such that $\matr{X} = \matr{Y} + \matr{Z}$, $\snorm{F}{\matr{Y}} + \snorm{F}{\matr{Z}} \leq \sqrt{2} \snorm{F}{\matr{X}}$ and $\snorm{2}{\matr{Y}} + \snorm{2}{\matr{Z}} \leq 2 \snorm{2}{\matr{X}}$.
\end{fact}
\begin{proof}
Since $\matr{X}$ is symemtric, we can always diagonalize it as $\matr{X} = \matr{Q} \matr{\Lambda} \matr{Q}^T$. Then, we rewrite $\matr{\Lambda} = \matr{\Lambda}^+  + \matr{\Lambda}^-$ where $\matr{\Lambda}^+$ contains only positive diagonal elements and $\matr{\Lambda}^-$ contains only negative diagonal elements. Set $\matr{Y} = \matr{Q} \matr{\Lambda}^+ \matr{Q}^T$ and $\matr{Z} = \matr{Q} \matr{\Lambda}^- \matr{Q}^T$. For $\ell 2$ norm, as all the eigenvalues of $\matr{Y}$ and $\matr{Z}$ comes from the eigenvalues of $\matr{X}$, the inequality is obvious. Then, for the Frobenius norm, it is easy to see that $\snorm{F}{\matr{Y}}^2 + \snorm{F}{\matr{Z}}^2 = \snorm{F}{\matr{X}}^2$. Since $\lp( \snorm{F}{\matr{Y}} + \snorm{F}{\matr{Z}} \rp)^2 \leq 2 \snorm{F}{\matr{Y}}^2 + 2 \snorm{F}{\matr{Z}}^2$, the fact follows.
\end{proof}
\begin{fact} \label{fact:abs_expectation}
For any matrix $\matr{A} \in \R^{d \times d}$ it holds $\E_{\vec{x} \sim \normal} \abs{\vec{x}^T \matr{A} \vec{x}} \leq \sqrt{2} \snorm{F}{A}$
\end{fact}
\begin{proof}
Recall that we defined the symmetrization of matrix $\matr A$ as $\mathcal S(\matr A) = \matr A + \matr A^T$.
We first replace $\matr{A}$ with $\frac{1}{2}\mathcal{S}(\matr{A})$. Then, use Fact \ref{fact:sym_decompose} to rewrite $\mathcal{S}(\matr{A})$ as the sum of two  definite matrices $\matr{A}_1$ and $\matr{A}_2$, where $\snorm{F}{\matr{A}_1} + \snorm{F}{\matr{A}_2} \leq \sqrt{2} \snorm{F}{\matr{A}}$. Then, it holds
\begin{align*}
\E_{\vec{x} \sim \normal} \abs{ \vec{x}^T \matr{A} \vec{x} }
=& \frac{1}{2} \E_{\vec{x} \sim \normal} \abs{\vec{x}^T \mathcal{S}(\matr{A}) \vec{x}}
\leq \frac{1}{2} \lp(\E_{\vec{x} \sim \normal} \vec{x}^T \matr{A}_1 \vec{x} + \E_{\vec{x} \sim \normal} \vec{x}^T \lp( -\matr{A}_2 \rp) \vec{x}\rp)\\
\leq& \frac{1}{2} \lp( \snorm{F}{\matr{A}_1} + \snorm{F}{\matr{A}_2} \rp)
\leq \sqrt{2} \snorm{F}{\matr{A}}
\end{align*}
\end{proof}
We first compute the first and second order derivatives of Discriminator's objective function $\LD$ with respect to $\matr A,b$. Given that the activation function $\phi$ is invertible on the set $T$,
recall that the Discriminator objective function is given by
\begin{align*}
\LD \lp( \matr A, b, \matr W \rp)
&=
\E_{\vec x \sim p{(\SIGMA^{1/2}, \phi)} }
\log( D(\vec x;\matr A, b) ) \nonumber
+
\E_{\vec x \sim p{(\matr W, \phi)}}
\log( 1 - D(\vec x;\matr A, b) )  \\
&=
\E_{\vec x \sim \normal(\SIGMA^{1/2}) }
\log( D(\phi(\vec x);\matr A, b) ) \nonumber
+
\E_{\vec x \sim \normal(\matr W)}
\log( 1 - D(\phi(\vec x);\matr A, b) ) \, ,
\end{align*}
where
$
D(\vec x; \matr A, b) =
\1_S(\vec x) \sigma \lp( \phi^{-1}( \vec x)^T \matr A \phi^{-1}( \vec x) + b\rp)  \hspace{-.6mm} + \hspace{-.6mm} \1_{S^c}(\vec x)/2
$ and $S = \phi(T)$. \\
Applying the chain rule then gives us
\begin{align} \label{eq:discriminator_derivative}
\nabla_{\matr{A},b} \LD(\matr{A},b;\matr{W})
= &
-\E_{\vec{x} \sim \normal(\matr{W})} \lp[
f' (h(\vec{x};\matr{A},b))
\begin{pmatrix}
(\XX)^{\flat} \\
1 \\
\end{pmatrix}
\1\{ \vec x \in T\}
\rp] \nonumber
\\ & -
\E_{\vec{x} \sim \normal(\SIGMA^{1/2})} \lp[
f' (h(\vec{x};\matr{A},b))
\begin{pmatrix}
(\XX)^{\flat} \\
1 \\
\end{pmatrix}
\1\{ \vec x \in T\}
\rp]
\nonumber \\
&+
\E_{\vec{x} \sim \normal(\SIGMA^{1/2})} \lp[
\begin{pmatrix}
(\XX)^{\flat} \\
1 \\
\end{pmatrix}
\1\{ \vec x \in T\}
\rp]
\end{align}
\begin{align} \label{eq:discriminator_hessian}
\nabla_{\matr{A},b}^2 \LD(\matr{A},b;\matr{W})= &-\E_{\vec{x} \sim \normal}  \lp[ f^{''} ( h(\matr{W}\vec{x};\matr{A},b) )
q(\matr{W}\vec{x})
\1\{ \matr{W} \vec{x} \in T\}
\rp] \nonumber
\\ & -
\E_{\vec{x} \sim \normal}  \lp[ f^{''} ( h(\SIGMA^{1/2}\vec{x};\matr{A},b) )
q(\SIGMA^{1/2} \vec{x})
\1\{\SIGMA^{1/2} \vec{x} \in T\}
\rp] \, ,
\end{align}
where for convenience we denote
$
q(\vec{x}) =
\begin{pmatrix}
(\XX)^{\flat} \\
1 \\
\end{pmatrix}
\otimes
\begin{pmatrix}
(\XX)^{\flat} \\
1 \\
\end{pmatrix}
$. \\
As the two terms of the Hessian above are similar, we will show how to handle the second term containing $\SIGMA^{1/2}$.
The other case follows similarly.
More specifically, for any $\vec{z} \in \R^{d^2}$ such that $\snorm{2}{\vec z} = 1$, we will show that
$$
\vec{z}^T
\E_{\vec{x} \sim \normal}  \lp[ f^{''} ( h(\SIGMA^{1/2}\vec{x};\matr{A},b) )
q(\SIGMA^{1/2} \vec{x})
\1\{\SIGMA^{1/2} \vec{x} \in T\}
\rp]
\vec{z}
$$
is bounded from below by some constant $\Omega_c(1)$ (that depends only on $c$).
First, notice that $f^{''}(y) = e^y/(e^y+1)^2$ is a positive, even function and is strictly decreasing when $y>0$. Thus, we must have $f^{''} (h(\SIGMA^{1/2}\vec{x};\matr{A},b)) \geq f^{''}(r)$ when $\abs{h(\SIGMA^{1/2}\vec{x};\matr{A}, b)} \leq r$. Besides, as $q(\SIGMA^{1/2}\vec{x})$ gives rise to a positive semi-definite matrix, we always have $\vec{z}^T q(\SIGMA^{1/2}\vec{x}) \vec{z} = \abs{\vec{z}^T q(\SIGMA^{1/2}\vec{x}) \vec{z} }$.
Hence, we proceed by defining two sets $P$ and $Q$ where the values of $f^{''}(\cdot)$ and $q(\cdot)$ are lower bounded respectively. Let $P = \{ \vec{x} \in \R^d: \, \abs{h(\SIGMA^{1/2} \vec{x};\matr{A}, b)} \leq r\}$ and $Q = \{ \vec{x} \in \R^d: \, \abs{\vec{z}^T q(\SIGMA^{1/2} \vec{x}) \vec{z}} \geq \gamma \}$. It then holds
\begin{align*}
& \vec{z}^T \E_{\vec{x} \sim \normal} \lp[ f^{''} (h(\SIGMA^{1/2}\vec{x}; \matr{A},b)) q(\SIGMA^{1/2}\vec{x}) \1\{\SIGMA^{1/2} \vec{x} \in T\} \rp] \vec{z} \\&
\geq
f^{''}(r) \gamma \E_{\vec x \in \normal} \lp[ \1\{\vec x \in Q\} \1\{\vec x \in P\} \1\{\SIGMA^{1/2} \vec{x} \in T\}\rp]
\end{align*}
Then, we lower bound the mass of each set.
For set $Q$, we can use the Gaussian anti-concentration of polynomials,
Lemma~\ref{thm:cabrey}, for the degree $4$ polynomial $\vec{z}^T q(\SIGMA^{1/2} \vec{x}) \vec{z}$ with respect to $\vec x$. We choose
$$
\gamma = \alpha^4 \E_{\vec{x} \sim \normal} \left[ \abs{\vec z^T q( \SIGMA^{1/2} \vec{x}) \vec z} \right]/ (16C)^4,
$$
where $\alpha := \normal(T;\SIGMA^{1/2}) = \Omega(1)$ as defined in Definition \ref{def:invertible_network} and $C = O(1)$
is the absolute constant defined in Lemma \ref{thm:cabrey}. It then holds $\normal(Q;\matr I) = Pr_{\vec x \sim \normal} \left[ \vec z^T q( \SIGMA^{1/2} \vec{x}) \vec z \leq \gamma \right] \leq \frac{\alpha}{4} $.
Moreover, for this specific choice of $\gamma$, it holds
\begin{align} \label{eq:gamma_bound}
\gamma &= \frac{\alpha^4}{(16C)^4} \E_{x \sim \normal } \left[
\vec z^T
\begin{pmatrix}
(\SIGMA^{1/2} \XX \SIGMA^{1/2})^{\flat} \\
1 \\
\end{pmatrix}
\otimes
\begin{pmatrix}
(\SIGMA^{1/2} \XX \SIGMA^{1/2})^{\flat} \\
1 \\
\end{pmatrix}
\vec z
\right] \nonumber \\
&\geq \lambda_{\min}\lp( \SIGMA \rp)^2 \frac{\alpha^4}{(16C )^4}
\geq \Omega_c(1) \, ,
\end{align}
since $\snorm{2}{\SIGMA^{-1}} \leq c + 1$ by Assumption \ref{initialization}. \\
Next, we will use Markov's Inequality to lower bound $ \Pr_{\vec x \sim \normal } \lp[ \abs{h(\SIGMA^{1/2}\vec{x};\matr{A},b)} > r \rp]$. We will first derive an upper bound for the expected value $\E_{\vec{x} \sim \normal} \lp[ \abs{h(\SIGMA^{1/2}\vec{x}; \matr{A}, b)} \rp]$. In particular, by the definition of Discriminator's projection set in Equation \eqref{eq:proj_d} ($\snorm{F}{\matr A} \leq \poly(c) $) and the constraint $\snorm{2}{\SIGMA} \leq 1 + c$ as implied in Assumption \ref{initialization}, it follows from Fact \ref{fact:abs_expectation} that
\begin{align*}
\E_{\vec{x} \sim \normal} \lp[ \abs{h(\SIGMA^{1/2}\vec{x}; \matr{A}, b)} \rp]
\leq
\sqrt{2} \snorm{F}{\SIGMA^{1/2}\matr{A}\SIGMA^{1/2}} + \abs{b}
\leq \poly(c).
\end{align*}
By Markov's inequality, we have
$\Pr_{\vec x \sim \normal} \left[ \abs{h(\SIGMA^{1/2}\vec{x};\matr{A},b)} > r \right] < \E_{\vec{x} \sim \normal} \left[  \abs{h(\SIGMA^{1/2}\vec{x};\matr{A},b)} \right] \cdot \frac{1}{r}$.
By setting $r = \Omega_c(1/\alpha)$, we obtain $\normal \left( \abs{h(\vec{x};\matr{A},b)} > r \right) < \frac{\alpha}{4}$.
Using the union bound, we then have
$$
\E_{\vec x \sim \normal}(\1\{ \vec x \in P\} \1\{\vec x \in Q\} \1\{\SIGMA^{1/2} \vec x \in T\}) \geq \alpha - \normal(\bar P) - \normal(\bar Q) \geq \alpha/2.
$$
Overall, we get
\begin{align*}
\vec{z}^T \E_{\vec{x} \sim \normal} \lp[ f^{''} (h(\matr{W}\vec{x}; \matr{A},b)) q(\matr{W}\vec{x}) \1\{\SIGMA^{1/2} \vec{x} \in T\} \rp] \vec{z}
\geq \frac{\alpha}{2} f^{''}(r) \gamma.
\end{align*}
Recall that we have set $r = \Omega_c(1/\alpha)$, $\gamma = \Omega_c(1)$ as shown in Equation \eqref{eq:gamma_bound}.
Since $f^{''}(r) = \frac{e^r}{(e^r+1)^2} \geq e^{-r}/4$ and $\alpha$ is an absolute constant, we conclude $\LD(\matr A, b;\matr W)$ is at least $\Omega_c(1)$ strongly concave.
\subsection{High Probability Projected Stochastic Gradient Descent}
The main tool used is Theorem 3 in the work of \cite{harvey2019simple}, which gives tight convergence rates of Projected Stochastic Gradient Descent in the high probability regime.
In their work, the theorem is proved for gradient oracle with sub-gaussian noise but it is not hard to see that the
statement holds for sub-exponential noise as well. We state the more general version of the theorem and provide its proof sketch.

To be consistent with the notation used in the original proof, here we will use subscript to denote iterations rather than index of elements from vector/matrix. Also, since we won't use the notion of "invertible region" in the section, we will use $T$ to denote the total number of iterations just in this section.
\begin{lemma}[High probability Projected Stochastic Gradient Descent, \cite{harvey2019simple}, Theorem C.12] \label{lem:high_prob_sgd}
Let $\mathcal Q$ be a convex set and $f: \mathcal Q \mapsto \R$  be $\mu$-strongly convex and $L$-Lipchitz
function with minimizer $\vec x_*$.
Moreover, let $\vec g_t$ be an unbiased stochastic gradient oracle of $f$.
Define $\eta_t = \frac{2}{\mu(t+1)}$.
Let $\vec x_t = \proj_{\mathcal Q}(\vec x_{t-1} - \eta_t \vec g_t)$.
Assume the following holds.
\begin{enumerate}[label=(\alph*)]
    \item There exists a constant $\tau$ such that $\sum_{t=1}^T \snorm{2}{\vec g_t}^2 \leq \wt O\big( \tau T \log^2(1/\delta)\big)$ with probability $1 - \delta$.
    \item There exists a pair of constants $\kappa$ and $\zeta$ such that for any $\lambda \in (0, 1/\zeta)$, \\
    we have $\E \lp[ \exp(\lambda \langle \vec g_t - \E[\vec g_t], \vec x_t - \vec x_*  \rangle ) \rp] \leq \exp \lp( \lambda^2 \kappa \snorm{2}{\vec x_t - \vec x_*}^2 \rp)$, where the expectation is conditional on $\vec g_{t-1}, \ldots, \vec g_{1}$.
\end{enumerate}
Let $\gamma_t = \frac{t}{T(T+1)/2}$.
Then, for any $\delta \in (0,1)$, with probability at least $1 - \delta$, it holds
$$
f\lp( \sum_{t=1}^T \gamma_t \vec x_t \rp) - f(\vec x_*) \leq \wt O \lp( \frac{ \lp( L^2 + \tau + \kappa + \zeta \rp) }{\mu} \frac{\log^2(1/\delta)}{T} \rp).
$$
\end{lemma}
We can follow the original proof of Theorem C.12 in \cite{harvey2019simple} until we reach the inequality
\begin{align} \label{eq:c5}
    f\lp( \sum_{t=1}^T \gamma_t \vec x_t \rp) - f(\vec x_*) \leq
    \frac{2}{T(T+1)} \sum_{t=1}^T t \cdot \langle \vec g_t - \E[\vec g_t], \vec x_t - \vec x_*  \rangle + \frac{2}{\mu T (T+1)} \sum_{t=1}^T \snorm{2}{\vec g_t}^2.
\end{align}
The original proof exploits the properties of sub-gaussian noise to bound the two summation sequences respectively (Lemma C.4 and C.5 in the original work). We state the two supporting lemmas and show they still hold in our scenario.
\begin{lemma}[Lemma C.4] \label{lem:c4}
For any $\delta \in (0,1)$, $\sum_{t=1}^T \snorm{2}{  {\vec g_t} }^2 = \wt O \lp(  \tau \cdot  T \cdot \log^2(1/\delta) \rp)$ with probability at least $1-\delta$.
\end{lemma}
\begin{proof}
The lemma is trivially true due to property (a).
\end{proof}
\begin{lemma}[Lemma C.5] \label{lem:c5}
Let $Z_T = \sum_{t=1}^T t \cdot \langle \vec g_t - \E \lp[ \vec g_t\rp], \vec x_t - \vec x_* \rangle$. Then, for any $\delta \in (0,1)$, we have $Z_T =  \wt O \lp( \frac{L^2 + \tau + \kappa + \zeta}{\mu} \cdot T \cdot \log(1/\delta) \rp)$ with probability at least $1 - \delta$.
\end{lemma}
The proof of Lemma \ref{lem:c5} relies on the following claims and lemma. We will follow the notation of the original proof and
define $d_t := t \cdot \langle \vec g_t - \E [ \vec g_t ], \vec x_t - \vec x_* \rangle$,
$
v_{t-1} := 2 \cdot \kappa \cdot t^2 \cdot \snorm{2}{\vec x_t - \vec x_*}^2
$
, and
$V_T := \sum_{t=1}^T v_{t-1}$.
\begin{claim}[Claim C.9] \label{claim:c9}
For any $\lambda \in (0, 1/(\zeta \cdot T))$,
$\E \lp[ \exp(\lambda \cdot t \cdot d_t ) \rp] \leq \exp \lp( \lambda^2 v_{t-1}/2 \rp)$, where the expectation is conditional on $\vec g_{t-1}, \ldots, \vec g_{1}$.
\end{claim}
\begin{proof}
The claim follows by scaling both sides of property (b) by a factor of $t \leq T$.
\end{proof}
\begin{claim}[Lemma C.11] \label{claim:c11}
There exists non-negative constants
$\alpha_1, \cdots, \alpha_T = O \lp(  \kappa \cdot \frac{T}{\mu} \rp) $, and
$\beta = \wt O \lp( \kappa \cdot (L^2 + \tau) \cdot \frac{T^2}{\mu^2} \rp)$
such that for every $\delta \in (0, 1)$, $V_T \leq \sum_{t=1}^T \alpha_t d_t + \beta \log^2(1/\delta) $ with probability at least $1 - \delta$.
\end{claim}
\begin{proof}
We follow the original proof until we reach the inequality
$$
V_T \leq \sum_{t=1}^T \alpha_t d_t + \kappa \cdot \sum_{t=1}^T O(T) \cdot \snorm{2}{\vec g_t}^2 + O(\kappa L^2/\mu^2) \, ,
$$
where $\alpha_t = O(\kappa \cdot T/\mu)$. By property (a), we have $\sum_{t=1}^T \snorm{2}{\vec g_t}^2 \leq \wt O\big( \tau T \log^2(1/\delta)\big)$ with probability $1 - \delta$. Hence, overall,
$$
\beta := \frac{1}{\log^2(1/\delta)} \cdot \lp( \kappa \cdot \sum_{t=1}^T O(T) \cdot \snorm{2}{\vec g_t}^2 + O(\kappa L^2/\mu^2) \rp) \leq \wt O \lp( \kappa \cdot (L^2 + \tau) \cdot \frac{T^2}{\mu^2} \rp).
$$
\end{proof}
\begin{lemma}[Generalized Freedman, \cite{harvey2019tight}, Theorem 3.3] \label{lem:c12}
Let $\{d_t, \mathcal F_t\}_{t=1}^T$ be a martingale difference sequence. Suppose that, for $t \in [T], v_{t-1}$ are non-negative $\mathcal{F}_{t-1}$-measurable random variables satisfying
$
\E \lp[  \exp(\lambda d_t) | \mathcal F_{t-1}  \rp] \leq \exp \lp( \lambda^2 v_{t-1} / 2 \rp) \, ,
$
for any $\lambda \in (0, 1/\zeta')$. Let $Z_T = \sum_{t=1}^T d_t$ and $V_T = \sum_{t=1}^T v_{t-1}$. Suppose there exists $\alpha_1, \cdots, \alpha_T, \beta \geq 0$ such that for every $\delta \in (0,1)$,
$
V_T \leq \sum_{t=1}^T \alpha_t d_t + \beta \log^2(1/\delta)
$
with probability at least $1 - \delta$. Let $\alpha \geq \max_{t \in [T]} \alpha_t$. Then
\begin{align*}
    \Pr[Z_T \geq x] \leq
    \exp \lp(  - \frac{x^2}{4 \alpha x + 2 \zeta' x + 8 \beta \log^2(1/\delta)} \rp) + \delta.
\end{align*}
\end{lemma}
\begin{proof}[Proof Sketch]
The original theorem does not have the constrain on $\lambda$. Nevertheless, the proof only requires $\E \lp[  \exp(\lambda d_t) | \mathcal F_{t-1}  \rp] \leq \exp \lp( \lambda^2 v_{t-1}/2 \rp)$ to hold for $\lambda$ upper bounded by some constant $C$. In fact,  $\lambda$ is fixed to be in the interval $[0,1 / (2 \alpha)]$ at the beginning of the proof. Hence, if $\zeta' \leq 2\alpha$, the original proof holds. In general, we can fix $\lambda \leq \min ( 1/(2 \alpha), 1/\zeta')$. Then, we can follow the original proof until we reach the inequality
$$
\Pr \lp [  Z_T \geq x \text{ and } V_T \leq \sum_{i=1}^T \alpha_i d_i + \beta \log^2(1/\delta)  \rp] \leq \exp(- \lambda(x - 2 \lambda \beta \log^2(1/\delta))).
$$
Instead of picking $\lambda = 1/{ \lp( 2 \alpha + 4 \beta \cdot \log^2(1/\delta) /x \rp)} \leq \frac{1}{2 \alpha}$, we pick
$$
\lambda = 1/\lp( {2 \alpha + \zeta' + 4 \beta \cdot \log^2(1/\delta) /x} \rp) \leq \min \lp( \frac{1}{2 \alpha}, \frac{1}{\zeta'} \rp).
$$ For this specific choice of $\lambda$, we conclude
$$
\Pr \lp [  Z_T \geq x \text{ and } V_T \leq \sum_{i=1}^T \alpha_i d_i + \beta \log^2(1/\delta) \rp] \leq \exp \lp(  - \frac{x^2}{4 \alpha x + 2 \zeta' x + 8 \beta \log^2(1/\delta)} \rp).
$$
Then, the result follows by applying union bounds on the events $ \{Z_t \geq x \}$ and $\{  V_T \leq \sum_{t=1}^T \alpha_i d_i + \beta \cdot \log^2(1/\delta)\}$.
\end{proof}

\begin{proof}[Proof of Lemma \ref{lem:c5}]
Claim \ref{claim:c9} shows that
$
\E \lp[  \exp(\lambda d_t) | \mathcal F_{t-1}  \rp] \leq \exp \lp( \lambda^2 v_{t-1} / 2 \rp) \, ,
$
for all $\lambda \in ( 0, 1/(\zeta T) )$.
By Claim \ref{claim:c11}, we have for every $\delta \in (0,1)$,
$
V_T \leq \sum_{t=1}^T \alpha_t d_t + \beta \log^2(1/\delta)
$. Hence, we can plugin $\alpha := \max_t \alpha_t = O\lp(\kappa \cdot \frac{T}{\mu} \rp)$, $\beta := \wt O \lp( \kappa \cdot (L^2 + \tau) \cdot \frac{T^2}{\mu^2} \rp)$ and $\zeta' := \zeta \cdot T$ into Lemma \ref{lem:c12}. Furthermore, we can set
\begin{align*}
x &= \wt \Theta \lp( \alpha + \zeta' +  \sqrt{\beta \log^2(1/\delta)} \rp)
    \\ &\leq \wt \Theta \lp( \kappa \cdot \frac{T}{\mu} + \zeta \cdot T + \lp( \kappa + L^2 + \tau \rp) \cdot \frac{T}{\mu} \log(1/\delta) \rp)
    \\ &\leq \wt \Theta \lp(  (\kappa + L^2 + \tau + \zeta) \cdot \frac{T}{\mu} \log(1/\delta) \rp) \, ,
\end{align*}
where in the first inequality we use the fact $\sqrt{ab + ac} \leq O(a + b + c)$ for all $a,b,c > 0$.
This then gives
$
    \Pr[Z_T \geq x] \leq 2 \cdot \delta.
$
We then get the statement in Lemma \ref{lem:c5} by rescaling the failing probability.
\end{proof}
\begin{proof}[Proof of Lemma \ref{lem:high_prob_sgd}]
Substituting the bounds obtained from Lemmas \ref{lem:c4} and \ref{lem:c5} into Equation \eqref{eq:c5} then gives
\begin{align*}
    f(\sum_{t=1}^t \gamma_t \vec x_t) - f(\vec x_*)
    &\leq \wt O \lp ( \frac{L^2 + \tau + \kappa + \zeta}{T(T+1)} \cdot \frac{T}{\mu} \cdot \log(1/\delta) \rp )
    + \wt O \lp( \tau \cdot T \cdot \log^2(1/\delta)\rp) \\
    & \leq \wt O\lp ( \frac{\lp(L^2 + \tau + \kappa + \zeta \rp)}{\mu} \cdot \frac{\log^2(1/\delta)}{T} \rp).
\end{align*}

\end{proof}
\subsection{Proof of Proposition \ref{prop:discriminator_convergence}}
We then proceed to show that our gradient oracle for Discriminator does satisfy conditions (a) and (b) in Lemma \ref{lem:high_prob_sgd}. We need the following well-known result on concentration of polynomials of independent
Gaussian random variables. See, e.g., \cite{o2014analysis}.
\begin{lemma}[Gaussian Hypercontractivity] \label{lem:hyper}
  Let  $h(\vec x): \R^d \mapsto \R$ be a degree-m polynomial. Then,
  $$
  \Pr_{\vec x \sim \normal} \lp[ \abs{ h(\vec x) - \E_{\vec y \sim \normal} \lp[ h(\vec y) \rp] } \geq l\rp]
  \leq
  e^2 \exp \lp( - \lp( \frac{l^2}{C \Var_{\vec x \sim \normal}\lp[ h(\vec x) \rp] } \rp)^{1/m} \rp) \, ,
  $$
  where $C>0$ is an absolute constant.
\end{lemma}

For our case, following Algorithm \ref{alg:simplified}, the gradient oracle used by the Discriminator is given by
\begin{align} \label{eq:discriminator_gradient_oracle}
\vec g^{(t)} &= \1\{ \vec x^{(t)}\in T \} \frac{1}{1 + e^{k}} \cdot \flatten{\vec x^{(t)} {\vec x^{(t)}}^T}{1} + \1\{\vec y^{(t)} \in T \} \frac{1}{1 + e^{-q}} \cdot  \flatten{\vec y^{(t)} {\vec y^{(t)}}^T}{1},
\end{align}
where $k = {\vec x^{(t)}}^T \matr{A}^{(t)} \vec x^{(t)} + b^{(t)}, q = {\vec y^{(t)}}^T \matr{A} \vec{y}^{(t)} + b^{(t)} \, ,\vec x^{(t)} \sim \normal(\matr W) \, , \vec y^{(t)} \sim \normal(\SIGMA^{1/2})$. The next lemma shows that this specific gradient oracle satisfies condition (a) required by Lemma \ref{lem:high_prob_sgd}.
\begin{lemma} \label{lem:bounded_sum_square}
For any $\delta \in (0,1)$, we have $\sum_{t=1}^{M_{\DIS}} \snorm{2}{ {\vec g^{(t)}} }^2 = \wt O_c( d^2 M_{\DIS} \log^2(1/\delta) )$ with probability at least $1 - \delta$.
\end{lemma}
\begin{proof}
First, notice that the sum is upper bounded by
\begin{align*}
    \sum_{t=1}^{M_{\DIS}} \snorm{2}{ \vec g^{(t)} }^2
    \leq
    \sum_{t=1}^{M_{\DIS}} \snorm{F}{ \vec x^{(t)} {\vec x^{(t)}}^T }^2
    +
    \sum_{t=1}^{M_{\DIS}} \snorm{F}{ \vec y^{(t)} {\vec y^{(t)}}^T }^2
    + 2 M_{\DIS}.
\end{align*}
Next, we will upper bound the expectation and variance of $\snorm{F}{\vec x^{(t)} {\vec x^{(t)}}^T}^2$ respectively (the bounds for $\snorm{F}{\vec y^{(t)} {\vec y^{(t)}}^T}^2$ can be obtained similarly). Recall that $\vec x^{(t)} \sim \normal(\matr W)$. As $\snorm{2}{\matr W} \leq \poly(c)$ by definition of the projection set, we can assume that
the data are generated by a standard normal, i.e., $\vec x \sim \normal(\matr I)$ by only losing a $\poly(c)$ factor in the upper bound. Hence, we have
$$
\E_{ {\vec x} \sim \normal(\matr I)} \lp[ \snorm{F}{ {\vec x} {\vec x}}^2 \rp] = \E_{\vec x \sim \normal(\matr I)} \lp[ \sum_{i,j} \vec x_i^2 \vec x_j^2  \rp]
\leq O(d^2) \, ,
$$
where the last inequality follows from standard bounds of moments of normal variables.
On the other hand,
\begin{align*}
    \Var_{ {\vec x} \sim \normal(\matr I)} \lp[ \snorm{F}{ \vec x \vec x^T }^2 \rp]
    &=
    \E_{ {\vec x} \sim \normal(\matr I)} \lp[ \snorm{F}{ \vec x \vec x^T }^4 \rp]
    - \lp( \E_{ {\vec x} \sim \normal(\matr I)} \lp[ \snorm{F}{ \vec x \vec x^T }^2 \rp] \rp)^2 \\
    &\leq
    \E_{\vec x \sim \normal(\matr I)} \lp[
    \lp( \sum_{i,j} \vec x_i^2 \vec x_j^2 \rp)^2
    \rp]
    \leq O(d^4) \, ,
\end{align*}
where the last inequality comes from the fact that $\lp( \sum_{i,j} \vec x_i^2 \vec x_j^2 \rp)^2$ equals the sum of $d^4$ eighth moments of normal variables.

We can then apply Lemma \ref{lem:hyper} with $l = \Theta_c(d^2 \log^2(1/\delta'))$ on the degree-4 polynomial $\snorm{F}{\vec x^{(t)} {\vec x^{(t)}}^T}^2$ and obtain
\begin{align*}
\text{Pr} \lp[ \abs{ \snorm{F}{\vec x^{(t)} {\vec x^{(t)}}^T}^2 - \E \lp[ \snorm{F}{\vec x^{(t)} {\vec x^{(t)}}^T}^2 \rp]} \geq l \rp]
\leq
e^2 \exp \lp(-
\lp(
\frac{l^2}{O_c(d^4)}
\rp)^{1/4}
\rp)
\leq \delta'.
\end{align*}
This implies with probability at least $1 - \delta'$,
\begin{align} \label{eq:hyper_bound}
\snorm{2}{\vec g^{(t)}}^2 \leq \snorm{F}{\vec x^{(t)} {\vec x^{(t)}}^T}^2 \leq O_c(d^2 \log^2(1/\delta)).
\end{align}
If we choose $\delta' = \delta/M_{\DIS}$ and takes the union bound over
the events that Equation \eqref{eq:hyper_bound} fails to hold for some $t \in [M_{\DIS}]$, we have
$$
\sum_{t=1}^{M_{\DIS}} \snorm{F}{ \vec x^{(t)} {\vec x^{(t)}}^T }^2 \leq O_c(d^2 T \log^2(1/\delta))
$$
with probability at least $1 - \delta$.
With a similar argument on $\sum_{t=1}^{M_{\DIS}} \snorm{F}{ \vec y^{(t)} {\vec y^{(t)}}^T }^2$, the statement follows.
\end{proof}
 Define the noise of the gradient oracle as $ {\vec z^{(t)}} =  {\vec g^{(t)}} - \E \lp[  {\vec g^{(t)}} \rp]$ conditioned on $\vec g_{t-1}, \cdots , \vec g_1$. Denote $\vec \theta^{(t)} = \flatten{\matr A^{(t)}}{b^{(t)}}$ and $\vec \theta_* = \flatten{\matr A_*}{b_*}$. The next lemma ensures that the noise satisfy condition (b) required by Lemma \ref{lem:high_prob_sgd}.
\begin{lemma} \label{lem:sub-exponential}
For all $\lambda \leq \frac{1}{O_c(1)}$, $t \in [M_{\DIS}]$, it holds
$$
\E \lp[ \exp \big( \lambda  \cdot \langle {\vec z^{(t)}} \, , \vec \theta_* - \vec \theta^{(t)}\rangle \big)  \rp]
\leq \exp \lp( \lambda^2 O_c(d) \snorm{2}{ \vec \theta_* - \vec \theta^{(t)}  }^2   \rp)
$$
\end{lemma}
\begin{proof}
By definition of $\vec z^{(t)}$, we have
\begin{align} \label{eq:exp_ineq}
    \langle \vec z^{(t)}, \vec \theta_* - \vec \theta^{(t)} \rangle
    &\leq \abs{\langle \vec g^{(t)}, \vec \theta_* - \vec \theta^{(t)} \rangle}
    + \abs{\langle \E \lp[ \vec g^{(t)} \rp], \vec \theta_* - \vec \theta^{(t)} \rangle}
    \leq \abs{\langle \vec g^{(t)}, \vec \theta_* - \vec \theta^{(t)} \rangle} +
    O_c(\snorm{2}{\vec \theta_* - \vec \theta^{(t)}}) \nonumber \\
    &\leq
    \abs{{\vec x^{(t)}}^T \lp( \matr A^{(t)} - \matr A_*\rp) \vec x^{(t)}} +
    \abs{{\vec y^{(t)}}^T \lp( \matr A^{(t)} - \matr A_*\rp) \vec y^{(t)}} +
    O_c(\snorm{2}{\vec \theta_* - \vec \theta^{(t)}}) \nonumber \\
    &\leq
    O_c\lp(\snorm{2}{\vec x^{(t)}}^2 + \snorm{2}{\vec y^{(t)}}^2\rp) \cdot \snorm{2}{\vec \theta_* - \vec \theta^{(t)}} \nonumber \\
    &\leq
    O_c\lp(\snorm{2}{\wt{\vec x}}^2 + \snorm{2}{\wt{\vec y}}^2\rp) \cdot \snorm{2}{\vec \theta_* - \vec \theta^{(t)}} \, ,
\end{align}
where $\wt{\vec x} := \matr W^{-1} \vec x^{(t)}, \wt{\vec y} := \SIGMA^{-1/2} \vec y^{(t)}$ and the last inequality follows from the bounded $\ell_2$ norm of $\matr W^{-1}$ and $\SIGMA^{-1/2}$. Notice that $\snorm{2}{\wt{\vec x}}^2, \snorm{2}{\wt{\vec y}}^2$ are the sum of $2d$ independent random variables following the chi-squared distribution. Hence,
\begin{align} \label{eq:chi_square_sub_exponential}
    \E \lp[ \exp \lp( \lambda \lp( \snorm{2}{\wt{\vec x}}^2 + \snorm{2}{\wt{\vec y}}^2 \rp)  \rp) \rp]
    \leq
    \exp( \lambda^2 O(d))
\end{align}
for all $\lambda \leq 1/O(1)$. By Equation \eqref{eq:exp_ineq}, we have for all $\lambda > 0$
\begin{align*}
\E \lp[ \exp \big( \lambda  \cdot \langle {\vec z^{(t)}} \, , \vec \theta_* - \vec \theta^{(t)} \rangle \big)  \rp]
&\leq
\E \lp[ \exp \big( \lambda  \cdot  \snorm{2}{\vec \theta_* - \vec \theta^{(t)}} \cdot O_c \lp( \snorm{2}{\wt{\vec x}}^2 + \snorm{2}{\wt{\vec y}}^2 \rp) \big)  \rp].
\end{align*}
Recall that for any $\vec \theta \in \mathcal Q_D$, we have $\snorm{2}{\vec \theta_* - \vec \theta } \leq O_c(1)$. Hence, we can scale Equation \eqref{eq:chi_square_sub_exponential} by a factor of $\snorm{2}{\vec \theta - \vec \theta_*}^2$ and conclude for all $\lambda \in (0, 1/O_c(1))$
$$
\E \lp[ \exp \big( \lambda  \cdot \langle {\vec z^{(t)}} \, , \vec \theta_* - \vec \theta^{(t)} \rangle \big)  \rp]
\leq \exp \lp( \lambda^2 O_c(d) \snorm{2}{ \vec \theta_* - \vec \theta^{(t)}  }^2   \rp).
$$
\end{proof}
Now, we can substitute the bounds obtained into Lemma \ref{lem:high_prob_sgd} to finish the proof of Proposition \ref{prop:discriminator_convergence}. By Lemma \ref{lem:d-convexity}, the objective function $\LD$ is $\mu = \Omega_c(1)$ strongly concave; by Lemma \ref{lem:bounded_sum_square}, we have $\tau = O_c(d^2)$; by Lemma \ref{lem:sub-exponential}, we have $\kappa = O_c(d)$ and $\zeta = O_c(1)$. Hence, fixing the Generator parameter $\matr W \in \mathcal Q_G$, if we run the inner loop for $M_{\DIS}$ iterations, with probability at least $1 - \delta$, when the algorithm exits the inner loop, the parameters $\matr A, b$ satisfy
\begin{align*}
    \abs{ \LD(\matr A_*, b_*;\matr W)  - \LD(\matr A, b;\matr W) }
    \leq O_c\lp( d^2 \log^2(1/\delta) / M_{\DIS}  \rp).
\end{align*}
Again, by strong concavity, it holds $\snorm{F}{\matr A - \matr A_*} + \abs{b - b_*}\leq O_c(\sqrt{ d^2 \log^2(1/\delta)/M_{\DIS} })$. Setting $M_{\DIS} =O_c\lp( d^2 \log^2(1/\delta) / \eps^2 \rp)$, we then obtain the statement.
\section{Training the Generator}
For convenience, we will use the same notations of $h(\vec{x};\matr{A},b)$, $h(\vec{x}; \matr{W})$ and $f(y)$ as used in Section~\ref{sec:appendix_dis}.
\subsection{Proof of Lemma \ref{lem:standard_optimality}} \label{sec:standard_optimality}
The gradient of the Virtual Training Criteria $\VV(\matr W)$ is given as
\begin{align} \label{eq:standard-induced-gradient}
\nabla_{\matr{W}} \VV(\matr{W})
= \lp( \SIGMA^{-1} \matr{W} - \lp( \matr{W}^{-1} \rp)^{T} \rp) \E_{\vec{x} \sim \normal} \lp[ f^{'} (h(\matr{W}\vec{x};\matr{W})) \XX \1\{\matr W \matr x  \in T\} \rp].
\end{align}
Thus, the Frobenius Norm of the gradient can be lower bounded by
\begin{align*}
\snorm{F}{\nabla_{\matr{W}} \VV(\matr{W})}
\geq &
\snorm{F}{\lp( \SIGMA^{-1} \matr{W} - \lp( \matr{W}^{-1} \rp)^{T} \rp)}
\min_{ \snorm{2}{\vec{z}} = 1 }\E_{\vec{x} \sim \normal} \lp[ \1 \{\matr W \matr x  \in T\} f^{'}(h(\matr{W}\vec{x};\matr{W})) \lp( \vec z^T \vec x \rp)^2  \rp].
\end{align*}
We now try to lower bound $\min_{ \snorm{2}{\vec{z}} = 1 }\E_{\vec{x} \sim \normal} \lp[ \1 \{\matr W \matr x  \in T\} f^{'} (h(\matr{W}\vec{x};\matr{W})) \lp( \vec z^T \vec x \rp)^2   \rp]$. Using the property that $f^{'}(y) = \sigma(y)$ is positive and monotonically increasing, we get the following inequality.
\begin{align*}
&\E_{\vec{x} \sim \normal} \lp[ \1 \{\matr W \matr x  \in T\} f^{'} (h(\matr{W}\vec{x};\matr{W})) \lp( \vec z^T \vec x \rp)^2  \rp]
\\
&\geq f^{'}(r) \gamma \E_{\vec{x} \sim \normal} \lp[ \1\{ h(\matr{W}\vec{x};\matr{W}) \geq r \} \1\{ \lp( \vec z^T \vec x \rp)^2 \geq \gamma\}  \1 \{\matr W \matr x  \in T\} \rp].
\end{align*}
By Fact \ref{fact:abs_expectation} and Lemma \ref{lem:generator_projection_property}, we can upper bound $\E_{\vec{x} \sim \normal} \abs{h(\matr{W}\vec{x};\matr{W})}$ by
\begin{align*}
\E_{\vec{x} \sim \normal} \abs{h(\matr{W}\vec{x};\matr{W})}
&\leq \frac{1}{2} \E_{\vec{x} \sim \normal} \lp[ \abs{\vec{x}^T \lp( \matr{I} - \matr{W}^T \SIGMA^{-1} \matr{W} \rp) \vec{x}} \rp] + \frac{1}{2} \abs{\log \det \lp( \matr{W}^T \SIGMA^{-1} \matr{W} \rp)} \\
&\leq \frac{1}{\sqrt 2} \invfdist + \abs{\log \det \lp( \matr{W}^T \SIGMA^{-1/2}\rp)} \\
&\leq \poly(c).
\end{align*}
Thus, if we choose $r<0$, by Markov's Inequality, we have
\begin{align*}
& \normal \lp( h(\matr{W}\vec{x};\matr{W}) \leq r \rp)
\leq
\normal \lp( \abs{h(\matr{W}\vec{x};\matr{W})} \geq \abs{r} \rp)
\leq \frac{1}{\abs{r}} \E_{\vec{x} \sim \normal} \abs{h(\matr{W}\vec{x};\matr{W})}.
\end{align*}
By Lemma \ref{lem:non_trivial_mass}, the Generator's mass in the set $T$ is always lower bounded by some absolute constant $k_{c}$ that depends only on $c$. By setting $r= -\frac{4}{k_{c}} \E_{\vec{x} \sim \normal} \abs{h(\matr{W}\vec{x};\matr{W})}$, we have $\normal \lp( h(\matr W\vec{x};\matr W) \leq r \rp) \leq k_{c}/4$.
On the other hand, for the degree $2$ polynomial $\lp(\vec z^T \vec x\rp)^2$,
we can again use the Gaussian anti-concentration of polynomials (Lemma~\ref{thm:cabrey}).
We choose
$$\gamma = \frac{1}{2} \lp( \frac{k_{c}}{8C} \rp)^2 \E_{\vec{x} \sim \normal} \lp[ \lp(\vec z^T \vec x\rp)^2 \rp] \, ,$$
and therefore, we have $\normal\lp( \lp(\vec z^T \vec x\rp)^2 \leq \gamma \rp) \leq k_{c}/4$.
Thus, by Union Bound, we conclude
\begin{align*}
&\E_{\vec{x} \sim \normal}\lp[ \1\{ h(\matr{W}\vec{x};\matr{W}) \geq r \} \1\{ \lp( \vec z^T \vec x \rp)^2 \geq \gamma \} \1\{ \matr W \matr{x} \in T \} \rp] \\
&\geq k_c - k_c/4 - k_c4
\geq k_c/2.
\end{align*}
Using the inequality $f^{'}(y) = \frac{1}{1 + \exp(-y)}\geq e^y/2$ when $y<0$, we obtain the bound
\begin{align*}
\min_{ \snorm{2}{\vec{z}} = 1 }\E_{\vec{x} \sim \normal} \lp[ f^{'}(h(\matr{W}\vec{x};\matr{W})) \lp(\vec z^T \vec x\rp)^2 \rp]
\geq \Omega_{c}(1).
\end{align*}
Therefore, given $\snorm{F}{\nabla_{\matr{W}} \VV(\matr{W})} \leq \eps$, it holds
\begin{align*}
\snorm{F}{ \SIGMA^{-1/2} \lp( \matr{W} \matr{W}^T  \rp) \SIGMA^{-1/2} - I  }
\leq  \snorm{F}{\lp( \SIGMA^{-1} \matr{W} - \lp( \matr{W}^{-1} \rp)^{T} \rp)} c^2
\leq O_{c}(\eps).
\end{align*}
Using Pinsker's inequality (and the exact expression of Kullback-Leibler divergence for normal distributions) we have
\begin{align*}
\dtv &(\normal( {\matr W}), \normal(\SIGMA^{1/2}) \leq
\snorm{F}{ \SIGMA^{-1/2}  \lp(  {\matr W}  {\matr W}^T  \rp) \SIGMA^{-1/2} - \matr I}
\leq
O_c(\eps)\,.
\end{align*}
Using the data processing inequality it follows that the total variation distance between
the transformed distributions $p{(\matr{W}, \phi)}$, $p{(\matr W_*, \phi)}$
is small, i.e.,
$
 \dtv \lp( p{(\matr{W}, \phi)}, p{(\matr W_*, \phi)} \rp) \leq O_{c}(\eps) \,.
$

Next, we consider the case when $\matr W$ lies on the boundary of $\mathcal Q_G$. For convenience, denote $\matr X = \E_{\vec{x} \sim \normal} \lp[ \1\{\matr{W} \vec x \in T\} f^{'} (h(\matr{W}\vec{x};\matr{W})) \XX \rp]$. Then, the gradient can be written as
\begin{align}
\nabla_{\matr{W}} \VV(\matr{W})
=   \lp( \SIGMA^{-1} \matr{W} - \matr{W}^{-1}  \rp) \matr X.
\end{align}
Consider the Singular Value Decomposition $\matr W^T \SIGMA^{-1/2} = \matr U \matr \Lambda \matr V$.
Since $\matr W$ is a first order stationary point, it holds
\begin{align} \label{eq:inward_gradient}
 \langle \nabla \VV (\matr W), \matr W - \SIGMA^{1/2} \matr V^T  \matr U^T \rangle \leq \eps  \snorm{F}{\matr W - \SIGMA^{1/2} \matr V^T  \matr U^T}.
\end{align}
By expanding the inner product, we obtain the lower bound
\begin{align} \label{eq:trace_bound}
\langle \nabla \VV (\matr W), \matr W - \SIGMA^{1/2} \matr V^T  \matr U^T \rangle
&=
\text{Tr} \lp( \matr X  \lp( \matr W^T \SIGMA^{-1} - \matr{W}^{-1}  \rp)   \lp( \matr W - \SIGMA^{1/2}\matr V^T  \matr U^T \rp) \rp) \nonumber \\
&= \text{Tr} \lp( \matr X \lp( \matr W^T \SIGMA^{-1/2} - \matr W^{-1} \SIGMA^{1/2} \rp)  \lp( \SIGMA^{-1/2} \matr W - \matr V^T \matr U^T \rp) \rp) \nonumber \\
&= \text{Tr} \lp( \matr X \matr U \lp( \matr \Lambda - \matr \Lambda^{-1} \rp) \matr V   \matr V^T \lp( \matr \Lambda - \matr I \rp) \matr U^T \rp) \nonumber \\
&\geq \lambda_{\min}(\matr X) \text{Tr} \lp( \matr \Lambda^2 - \matr \Lambda - \matr I + \matr \Lambda^{-1} \rp).
\end{align}
From the discussion of the previous case where $\matr W$ is an interior point, we have $\lambda_{\min}(\matr X) \geq \Omega_c(1)$. Hence,
combining Equation \eqref{eq:trace_bound} with Equation \eqref{eq:inward_gradient} gives
$$ \text{Tr} \lp( \matr \Lambda^2 - \matr \Lambda - \matr I + \matr \Lambda^{-1} \rp) \leq O_c(\eps) \snorm{F}{\matr W - \SIGMA^{1/2} \matr V^T  \matr U^T}.$$
Notice that $\text{Tr} \lp( \matr \Lambda^2 - \matr \Lambda - \matr I + \matr \Lambda^{-1} \rp)$ is 2-strongly convex with respect to $\matr \Lambda$ and minimizes at $\matr \Lambda = \matr I$.\\
Hence, using convexity, we get
\begin{align} \label{eq:rhs}
\snorm{F}{\matr \Lambda - \matr I}^2 \leq   O_c(\eps) \snorm{F}{\matr W - \SIGMA^{1/2} \matr V^T  \matr U^T}.
\end{align}
On the other hand, it holds
\begin{align*}
\snorm{F}{\matr W - \SIGMA^{1/2} \matr V^T  \matr U^T} &=
\snorm{F}{ \lp( \matr W^T \SIGMA^{-1/2} - \matr U \matr V \rp) \SIGMA^{1/2}} \leq \snorm{F}{ \matr \Lambda - \matr I } \snorm{2}{\SIGMA^{1/2}}.
\end{align*}
By Assumption \ref{initialization}, we have $\snorm{2}{\SIGMA} \leq (1+c)$. Hence,
\begin{align} \label{eq:lhs}
    \snorm{F}{\matr \Lambda -\matr I} \geq \Omega_c(1) \snorm{F}{\matr W - \SIGMA^{1/2} \matr V^T  \matr U^T}.
\end{align}
Combining Equation \eqref{eq:rhs} and \eqref{eq:lhs}, we then have
$$
\snorm{F}{\matr W - \SIGMA^{1/2} \matr V^T  \matr U^T} \leq O_c(\eps).
$$
This therefore implies
\begin{align*}
\snorm{F}{\SIGMA^{-1/2}  \lp( \matr W \matr W^T \rp) \SIGMA^{-1/2} - \matr I } \leq O_c(\eps) \,
\end{align*}
since the expression is $O_c(1)$-Lipchitz with respect to $\matr W$ and the expressions evaluates to $0$ when $\matr W = \SIGMA^{1/2} \matr V^T  \matr U^T$. The rest of the proof is then identical to the case when $\matr W$ is an interior point.

\subsection{Proof of Lemma \ref{lem:standard_cross_lip}}
The actual gradient used in training takes a similar form as Equation \eqref{eq:standard-induced-gradient}. The difference is that the Discriminator has now parameters $\matr A,b$ instead of the optimal $\matr A_*, b_*$
In particular, the expected value of the training gradients are given by
\begin{align*}
\nabla_{\matr W} \LG \lp( \matr W; \matr A, b \rp)
=
2 \matr A \matr W \E_{\vec x \sim \normal} \lp[
f^{'}\lp( h \lp( \matr W \vec x; \matr A, b \rp) \rp) \1\{ \matr W \vec{x} \in T \} \vec x \vec x^T
\rp].
\end{align*}
We proceed to compute the expression's derivatives with respect to $\matr A$ and $b$. For $\matr A$, we have
\begin{align*}
\nabla_{\matr A} \nabla_{\matr W} \LG \lp( \matr W; \matr A, b \rp)
&= 2
\E_{\vec x \sim \normal} \lp[
f^{'}\lp( h \lp( \matr W \vec x; \matr A, b \rp) \rp) \1\{ \matr W \vec{x} \in T \}  \lp( \matr W \XX \matr W^T \rp) \otimes \lp( \matr A \matr W \XX \rp)
\rp] \\
&+
2
\E_{\vec x \sim \normal} \lp[
f^{''}\lp( h \lp( \matr W \vec x; \matr A, b \rp) \rp) \1\{ \matr W \vec{x} \in T \}
\lp( \XX \matr W^T \rp) \otimes \matr I
\rp].
\end{align*}
Notice that $\matr A$ is a symmetric matrix but it is not necessarily positive semi-definite. Nevertheless, using Fact \ref{fact:sym_decompose}, we can write it as $\matr A = \matr A^{-} + \matr A^{+}$, where $\matr A^{-}$ is negative semi-definite, $\matr A^{+}$ is positive semi-definite and both have their l2-norms bounded by $\snorm{2}{\matr A}$. Then, by splitting the expressions with triangle inequality, we can without loss of generality assume $\matr A$ is positive semi-definite by losing a constant factor in the upper bound.
Then, we replace all the non-negative scalar-valued $f^{'}(\cdot)$ and $f^{''}(\cdot)$ functions with their upper bound $1$. Lastly, using linearity of expectation, we take expectation over terms involving $\vec x$, which gives $\E_{\vec x \sim \normal(\matr I)} \lp[ \vec x \vec x^T \rp] = \matr I$.
Hence, we obtain
\begin{align*}
\snorm{2}{\nabla_{\matr A} \nabla_{\matr W} \LG \lp( \matr W; \matr A, b \rp)}
&\leq
4 \snorm{2}{ \lp( \matr W \matr W^T \rp) \otimes \lp( \matr A \matr W \rp) }
+
2 \snorm{2}{ \matr W^T \otimes \matr I }.
\end{align*}
Since $\snorm{2}{\matr A},\snorm{2}{\matr W} \leq \poly(c)$, the l2-norm is bounded above by $\poly(c)$.\\
For $b$, we have
$
\frac{\partial}{\partial b} \nabla_{\matr W}  \LG \lp( \matr W; \matr A, b \rp) =
2 \matr A \matr W \E_{\vec x \sim \normal} \lp[
f^{''}\lp( h \lp( \matr W \vec x; \matr A, b \rp) \rp) \1\{ \matr W \vec{x} \in T \} \vec x \vec x^T
\rp].
$
Similarly, the norm can be upper bounded by $\poly(c)$. Hence, overall, the training gradient is $O_c(1)$-Lipchitz with respect to $\matr A, b$.
\subsection{Proof of Lemma \ref{lem:standard_bounded_variance}} \label{sec:standard_bounded_variance}
For Generator, it is easy to see that the gradient oracle takes the form
\begin{align*}
\vec g_{\GEN} = \nabla_{\matr W} \log\lp(1 - D\lp( \phi(\matr W \vec x); \matr A, b \rp)\rp) = \1\{\matr W \vec x \in T\} f^{'}( h(\matr W \vec{x}; \matr A, \matr b) ) 2 \matr{A}\matr{W}\XX \, ,
\end{align*}
where $\vec x \sim \normal(\matr I)$.
Hence, we can upper bound the square $\ell_2$-norm of it by
\begin{align*}
\E_{\vec{x} \sim \normal} \lp[ \snorm{2}{\vec g_{\GEN}}^2 \rp]
& =
\E_{\vec{x} \sim \normal} \lp[ \snorm{F}{\1\{\matr W \vec x \in T\} f^{'}( h(\matr W \vec{x}; \matr A, \matr b) ) 2 \matr{A}\matr{W}\XX}^2 \rp]\\
& \leq
4 \snorm{2}{\matr{A}\matr{W}}^2 \E_{ \vec{x} \sim \normal } \snorm{F}{\XX}^2
\leq
4 \snorm{2}{\matr{A}\matr{W}}^2 \lp(d^2 + 2d\rp).
\end{align*}
By definition of the projection sets $\mathcal Q_G$ and $\mathcal Q_D$, we have $\snorm{2}{\matr A},\snorm{2}{\matr W} \leq \poly(c)$. Hence, it holds $\E_{\vec{x} \sim \normal} \lp[ \snorm{2}{\vec g_{\GEN}}^2 \rp] \leq O_c(d^2)$.

\subsection{Proof of Lemma \ref{lem:biased_strongly_convex}} \label{sec:biased_strongly_convex}
In this section, we prove our main optimization tool: a  lemma stating the convergence of Biased Stochastic Gradient Descent(BSGD). Technically
its proof is standard (similar to the work of~\cite{ghadimi2016mini}) and we provide it here for completeness. In the setting, we try to optimize a function $f(\vec x)$ when only a biased gradient estimator $\xi(\vec x)$ of $\nabla f(\vec x)$ is provided. In particular, we study the following Projected Stochastic Gradient Descent Algorithm under a convex set $\mathcal Q$.

\begin{algorithm}[H]
  \caption{Biased PSGD for $f(\bw)$}
  \label{alg:random_psgd}
  Procedure: BPSGD({$f, M, \beta$})
  \begin{algorithmic}[1]
    \State Sample the stopping time $m$ uniformly from $\{1,\ldots, M\}$
    \For{$i = 1, \dots, m$}
    \State Sample $\vec \xi^{(i)}$.
     $\triangleright \snorm{2}{\E[\vec \xi] - \nabla_{\vec w}f(\vec{w}^{(i)})} < \alpha$
    \State ${\vec w}^{(i+1/2)} \gets {\vec w}^{(i)} - \beta \vec \xi^{(i)}$.
    \State ${\vec w}^{(i+1)} \gets \argmin_{\vec w \in \mathcal{Q}}\snorm{2}{\vec w - \vec w^{(i+1/2)}}$.
    \EndFor
  \end{algorithmic}
\end{algorithm}
Notice that, since we do not require the objective function $f(\vec x)$ to be convex,
we can only guarantee convergence to stationary points of the objective function.

We will use the following Lemma which is standard for non-convex projected gradient descent.
\begin{lemma}
\label{lem:convex_projection}
Assume function $f$ is $L$-smooth. Consider the gradient mapping $g_{\mathcal{Q}}^{\eta}(\vec{w}, \nabla f (\vec{w})) = \frac{1}{\eta} (\vec{w} - \proj_{\mathcal{Q}}(\vec{w} - \eta \nabla f(\vec{w})))$. It holds $g^{\eta}_{\mathcal{Q}}(\vec{w}, \nabla f (\vec{w})) \geq \eps \frac{1}{1 + L \eta}$ if there exists $\vec{u} \in \mathcal{Q}$ such that $\frac{1}{ \snorm{2}{\bar{\vec{w}} - \vec{u}} }\langle \nabla f(\bar{\vec{w}}), \bar{\vec{w}} - \vec{u} \rangle \geq \eps$, where $\bar{\vec{w}} = \proj_{\mathcal{Q}}(\vec{w} - \eta \nabla f(\vec{w}))$.
\end{lemma}

\begin{proof}[Proof of Lemma \ref{lem:biased_strongly_convex}]
  Consider the update before the projection step
  $
  \vec w^{(i+1/2)} = {\vec w}^{(i)} - \beta  \vec \xi^{(i)}.
  $
  After the projection step we have
  $
  \vec w^{(i+1)} = \argmin_{\vec x \in \mathcal{Q}} \snorm{2}{ \vec w^{(i+1/2)} - \vec x}^2.
  $
  Denote the projection operator as
  $
   \vec p_\mathcal{Q}(\vec w, \vec \xi)
   = \argmin_{\vec x \in \mathcal{Q}} \snorm{2}{\vec x - \vec w - \beta \vec \xi }.
  $
  Besides, we define the gradient mapping on the convex set $\mathcal{Q}$ of point $\vec x$ to
  be $g_\mathcal{Q}(\vec w, \vec \xi) = (1/\beta) (\vec w - \vec p_\mathcal{Q}(\vec w, \vec \xi))$.
  It follows from standard arguments (for example, Theorem 1.2.3 of
  \cite{nesterov}) that
  \begin{align*}
  f(\vec w^{(i+1)}) - f(\vec w^{(i)})
  &\leq
  \nabla f(\vec w^{(i)})^T (\vec w^{(i+1)} - \vec w^{(i)})
  + \frac{L}{2} \snorm{2}{\vec w^{(i+1)} - \vec w^{(i)}}^2
  \\
  &\leq
  - \beta \nabla f(\vec w^{(i)})^T g_\mathcal{Q}(\vec w^{(i)}, \vec \xi)
  + \frac{L \beta^2}{2} \snorm{2}{g_\mathcal{Q}(\vec w^{(i)}, \vec \xi)}^2
  \\
  &\leq
  -\beta \nabla f(\vec w^{(i)})^T g_\mathcal{Q}(\vec w^{(i)}, \vec \xi)
  + \frac{L B \beta^2}{2}.
\end{align*}
Notice that that since $\vec \xi$ is a biased estimate of the gradient we
have $\E[g_\mathcal{Q}(\vec w, \vec \xi)] = g_\mathcal{Q}(\vec w, \nabla f(\vec w)) + \vec e$,
for some error vector $\vec e$ with $\snorm{2}{\vec e} \leq \alpha$. This is
true because in expectation the minimizer of $\snorm{2}{ \vec w^{(i+1/2)} - \vec x}$ only changes by $\vec e$.
Additionally, $f$ is $l$-Lipchitz.
Therefore, after taking the expectation conditional on $\vec w^{(i)}$,
we have
\begin{align*}
  \E[ f(\vec w^{(i+1)}) - f(\vec w^{(i)}) |\vec w^{(i)} ]
  &\leq -\beta \nabla f(\vec w^{(i)})^T g_\mathcal{Q}(\vec w^{(i)}, \nabla f(\vec w^{(i)}))
  + \frac{L B \beta^2}{2} + l \beta \alpha.
\end{align*}
Since we project onto a convex set $\mathcal{Q}$ we have that
$$
\nabla f(\vec w^{(i)})^T g_\mathcal{Q}(\vec w^{(i)}, \nabla f(\vec w^{(i)}))
\geq \snorm{2}{g_\mathcal{Q}(\vec w^{(i)}, \nabla f(\vec w^{(i)}))}^2.
$$
Therefore putting everything together we obtain,
$$
  \E[ f(\vec w^{(i+1)}) - f(\vec w^{(i)}) |\vec w^{(i)} ]
  \leq -\beta \snorm{2}{ g_\mathcal{Q}(\vec w^{(i)}, \nabla f(\vec w^{(i)}))}^2
  + \frac{L B \beta^2}{2}
  + l \beta \alpha.
$$
Rearranging, summing over $i=1,\ldots, M$, and using the law of total
expectation, we obtain
$$
\sum_{i=1}^M \snorm{2}{g_\mathcal{Q}^{(i)}}^2 \leq
\frac{R}{\beta} + \beta \frac{L B}{2} M + l \alpha M.
$$
Picking step size $\beta = \sqrt{2R/(LBM)}$ we obtain that
$
\sum_{i=1}^M \snorm{2}{g_\mathcal{Q}^{(i)}}^2 \leq
\sqrt{2RLBM} + l \alpha M.
$
Next we choose a random stopping time $m$ uniformly in $\{1,\ldots, M\}$,
where $M = O( RLB/ \eps^4)$.  We then have
$$
E\lp[ \snorm{2}{g_\mathcal{Q}^{(m)}}^2\rp]
= \frac{1}{M} \sum_{i=1}^M \snorm{2}{g_\mathcal{Q}^{(i)}}^2
\leq \sqrt{2RLB/M} + l \alpha
\leq O(\eps^2) + l \alpha
$$
From Markov's inequality we get that  with probability at least $99 \%$ it holds that when the SGD stops
we have $\snorm{2}{g_\mathcal{Q}^{(m)}} \leq \sqrt{O(\eps^2 + l \alpha)} \leq O(\eps + \sqrt{l \alpha})$. Then, applying Lemma \ref{lem:convex_projection}, it holds for any $\vec u \in \mathcal Q$,
$$
\frac{1}{\snorm{2}{ \vec w^{(m)} - \vec u }}
\langle
\nabla_{\vec w} f(\vec w^{(m)}),
\vec w^{(m)}- \vec u
\rangle
\leq  O(\eps + \sqrt{l \alpha}) (1 + L \beta)
\leq O(\eps + \sqrt{l \alpha}) \, ,
$$
where the last inequality follows from the fact that
$
L \beta = L \sqrt{\frac{R}{LBM}} = O( \eps^2/B^2 ) = o(1)
$.
\end{proof}

\subsection{Some additional properties required by Biased SGD}
As a standard requirement for Gradient Descent, we show that $\VV(\matr W)$ is locally smooth and Lipchitz-continuous.
\begin{lemma} \label{lem:standard_smooth}
$\VV(\matr{W})$ is $\poly(c)$-smooth and $\poly(c)$-Lipchitz continuous with respect to $\matr W$ in the projection set $\mathcal Q_G$ when Assumption \ref{initialization} is satisfied.
\end{lemma}
\begin{proof}
Recall that the gradient of the Virtual Training Criteria $\VV(\matr W)$ is given as
\begin{align*}
\nabla_{\matr{W}} \VV(\matr{W})
= \lp( \SIGMA^{-1} \matr{W} - \lp( \matr{W}^{-1} \rp)^{T} \rp) \E_{\vec{x} \sim \normal} \lp[ f^{'} (h(\matr{W}\vec{x};\matr{W})) \XX \1\{\matr W \matr x  \in T\} \rp]
\end{align*}
Since $f^{'}(\cdot)$ is a positive function upper bounded by $1$, it holds
\begin{align*}
    \snorm{2}{ \nabla_{\matr W} \VV(\matr W) } \leq \snorm{2}{ \SIGMA^{-1} \matr W - \invtrans{\matr W} }
\end{align*}
Since $\snorm{2}{\matr W}, \snorm{2}{\SIGMA} \leq \poly(c)$ by definition of the projection set and Assumption \ref{initialization}, we conclude $\VV(\matr W)$ is $\poly(c)$-Lipchitz continuous. \\
Next, we compute and upper bound the l2-norm of the hessian.
\begin{align*}
\nabla_{\matr{W}}^2 \VV(\matr{W})
& \hspace{-.5mm}= \hspace{-.5mm}
\E_{\vec{x} \sim \normal}
\bigg[
\1 \{ \matr W \vec x \in T\}f^{''} (h(\matr{W}\vec{x};\matr{W}))
\lp( \invtrans{\matr{W}} \hspace{-.25mm}-\hspace{-.25mm} \SIGMA^{-1} \matr{W} \XX \rp) \\
& \hspace{3.5em} \otimes
\lp( \SIGMA^{-1} \matr{W} \XX \hspace{-.25mm}- \hspace{-.25mm}\invtrans{\matr{W}} \XX \rp)
\bigg] \\
&+
\E_{\vec{x} \sim \normal}  \lp[ \1 \{ \matr W \vec x \in T\} f^{'} (h(\matr{W}\vec{x};\matr{W}))
\lp( \XX \otimes \SIGMA^{-1} + \lp( \XX \invtrans{\matr W} \rp) \otimes \matr W^{-1} \rp) \rp].
\end{align*}
We could then use triangle inequality to split the expressions into sum of positive definite matrices (in a sense that $\vec z^T \matr M \vec z \geq 0$ for any $\vec z$). Then, we replace all the non-negative scalar-valued $f^{'}(\cdot)$ and $f^{''}(\cdot)$ with their upper bound $1$. Lastly, by linearity of expectation, we take expectation over terms involving $\vec x$, which gives $\E_{\vec x \sim \normal(\matr I)} \lp[ \vec x \vec x^T \rp] = \matr I$.
Hence, we obtain the following upper bound
\begin{align*}
\snorm{2}{\nabla_{\matr{W}}^2 \VV(\matr{W})}
&\leq
\snorm{2}{\invtrans{\matr W} \otimes \lp( \SIGMA^{-1}  \matr W \rp) }
+ \snorm{2}{\invtrans{\matr W} \otimes \invtrans{\matr W}} + \snorm{2}{\lp( \SIGMA^{-1} \matr W \rp) \otimes \lp( \SIGMA^{-1} \matr W \rp)} \\
&+ \snorm{2}{\lp( \SIGMA^{-1}  \matr W \rp) \otimes \invtrans{\matr W}}  +
\snorm{2}{\matr I \otimes \SIGMA^{-1}} +
\snorm{2}{\invtrans{ \matr W } \otimes \matr W^{-1}}
\end{align*}
 Since the l2-norm of $\matr W$, $\matr W^{-1}$, $\SIGMA^{1/2}$ and $\SIGMA^{-1/2}$ are all bounded by $\poly(c)$, it then follows $\snorm{2}{\nabla_{\matr{W}}^2 \VV(\matr{W})} \leq \poly(c)$.
\end{proof}

 \end{document}